\documentclass[sigconf,natbib=false]{acmart}
\def\shownotes{1}
\usepackage{macro}
\usepackage[backend=biber,
            style=numeric-comp,
            bibstyle=numeric,
            natbib=true,
            maxbibnames=99,
            minalphanames=3,
            maxcitenames=2,
            sorting=ynt,
            giveninits=true,            
            labeldateparts=true,
            sortcites=true]{biblatex}
\usepackage{algorithm,algorithmic}
\usepackage{enumitem}
\usepackage{caption}
\usepackage{subcaption}
\usepackage{balance}
\newcommand{\mathbbm}[1]{\text{\usefont{U}{bbm}{m}{n}#1}}

\AtBeginDocument{%
 }

\copyrightyear{2023}
\acmYear{2023}
\setcopyright{rightsretained}
\acmConference[KDD '23]{Proceedings of the 29th ACM SIGKDD Conference on Knowledge Discovery and Data Mining}{August 6--10, 2023}{Long Beach, CA, USA}
\acmBooktitle{Proceedings of the 29th ACM SIGKDD Conference on Knowledge Discovery and Data Mining (KDD '23), August 6--10, 2023, Long Beach, CA, USA}
\acmDOI{10.1145/3580305.3599265}
\acmISBN{979-8-4007-0103-0/23/08}

\makeatletter
\gdef\@copyrightpermission{
  \begin{minipage}{0.3\columnwidth}
  \href{https://creativecommons.org/licenses/by/4.0/}{\includegraphics[width=0.90\textwidth]{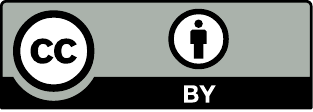}}
  \end{minipage}\hfill
   \begin{minipage}{0.7\columnwidth}
   \href{https://creativecommons.org/licenses/by/4.0/}{This work is licensed under a Creative Commons Attribution International 4.0 License.}
  \end{minipage}
  \vspace{5pt}
}
\makeatother

\begin{document}

\title[Boosting Multitask Learning on Graphs through Higher-Order Task Affinities]{Boosting Multitask Learning on Graphs through\\Higher-Order Task Affinities}

\author{Dongyue Li}
\affiliation{%
  \institution{Northeastern University}
  \country{}
}
\email{li.dongyu@northeastern.edu}

\author{Haotian Ju}
\affiliation{%
  \institution{Northeastern University}
  \country{}
}
\email{ju.h@northeastern.edu}

\author{Aneesh Sharma}
\affiliation{%
  \institution{Google}
  \country{}
}
\email{aneesh@google.com}

\author{Hongyang R. Zhang}
\affiliation{%
  \institution{Northeastern University}
  \country{}
}
\email{ho.zhang@northeastern.edu}

\renewcommand{\shortauthors}{Dongyue Li, Haotian Ju, Aneesh Sharma, \& Hongyang R. Zhang}

\begin{abstract}
Predicting node labels on a given graph is a widely studied problem with many applications, including community detection and molecular graph prediction. This paper considers predicting multiple node labeling functions on graphs simultaneously and revisits this problem from a multitask learning perspective. For a concrete example, consider overlapping community detection: each community membership is a binary node classification task. Due to complex overlapping patterns, we find that negative transfer is prevalent when we apply naive multitask learning to multiple community detection, as task relationships are highly nonlinear across different node labeling. To address the challenge, we develop an algorithm to cluster tasks into groups based on a higher-order task affinity measure. We then fit a multitask model on each task group, resulting in a boosting procedure on top of the baseline model. We estimate the higher-order task affinity measure between two tasks as the prediction loss of one task in the presence of another task and a random subset of other tasks. Then, we use spectral clustering on the affinity score matrix to identify task grouping. We design several speedup techniques to compute the higher-order affinity scores efficiently and show that they can predict negative transfers more accurately than pairwise task affinities. We validate our procedure using various community detection and molecular graph prediction data sets, showing favorable results compared with existing methods. Lastly, we provide a theoretical analysis to show that under a planted block model of tasks on graphs, our affinity scores can provably separate tasks into groups.
\end{abstract}




\maketitle
\hypersetup{linkcolor={ACMRed},citecolor={black},urlcolor={ACMRed}}

\section{Introduction}

\begin{figure*}[t!]
    \centering
    \begin{minipage}[b]{0.99\textwidth}
        \centering
        \includegraphics[width=0.99\textwidth]{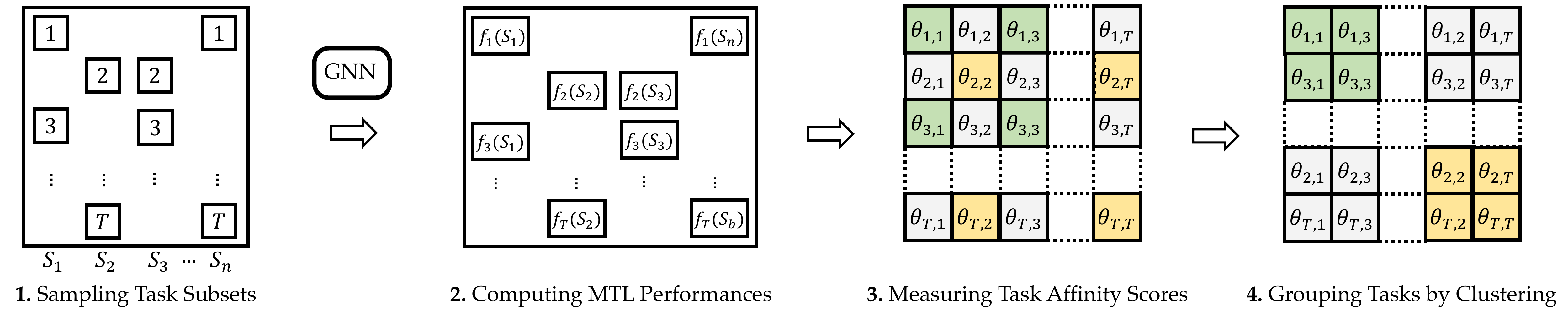}
    \end{minipage}
    \caption{Overview of our boosting procedure:
    (1) We sample random subsets of tasks, each subset containing a fixed number of tasks.
    (2) For each subset $S_k$, for $k = 1, 2, \dots, n$, we fit a multitask learning (MTL) model on the combined data sets of all tasks in $S_k$, using a graph neural network (GNN) as the shared encoder. After fitting the MTL model, we evaluate its prediction performance for each task $i \in S_k$, denoted as $f_i(S_k)$.
    (3) We compute an affinity score $\theta_{i, j}$ by averaging task $i$'s scores among all subsets as in equation \eqref{eq_aff}, where $n_{i, j}$ is the number of subsets including both $i, j$.
    This results in a $T$ by $T$ affinity matrix, denoted as $[\theta_{i, j}]_{T\times T}$.
    (4) We apply spectral clustering on this matrix to find clusters of task groups and fit one GNN for each task group.}
    \label{fig_pipeline}
\end{figure*}

Given multiple node label prediction tasks on graphs, can we design a model to predict all of them simultaneously?
For example, consider supervised, overlapping community detection \cite{chen2017supervised}: Given a set of labels as a seed set, can we learn to predict whether another node of the graph belongs to a community?
This is an example of multi-label learning: The input involves labels of multiple communities. The goal is to predict community labels for the remaining nodes of each community.
In this paper, we cast this multi-label classification problem into a more general formulation of multitask learning \cite{caruana1997multitask,zhang2021survey}, which has numerous other applications in knowledge graphs \cite{wang2019multi} and protein-protein interactions \cite{hu2019strategies}.
Traditionally, multitask learning works by fitting a single model on all tasks \citep{crammer2008learning,ben2010theory}.
By contrast, we investigate task relationships on graphs with a shared graph neural network. We observe notable negative transfers between various tasks.
We address this challenge by designing a boosting procedure to cluster tasks into groups and then fit multiple graph neural networks, one for each task group.

Negative transfer refers to scenarios where learning multiple tasks using a shared encoder can worsen performance compared to single-task learning (STL) for one task \cite{wu2020understanding}, and this has been observed for various data modalities \cite{yang2016deep,vu2020exploring,wu2020generalization,li2021improved}.
The cause can often be attributed to heterogeneity in input features or dataset sizes.
Much less is known about graph data.
\citet{zhu2021transfer} and \citet{ju2023generalization} identify structural differences for transfer learning on graphs due to the graph diffusion process.
These structural differences can also cause negative interference in multitask learning.
Thus, building multitask graph neural networks requires better modeling of structural differences between tasks.
This modeling could also be used to identify task grouping \cite{fifty2021efficiently}.

The classical multitask learning literature has studied task relatedness through heuristics and theoretical perspectives \cite{caruana1997multitask,ben2003exploiting,ben2010theory}. These results do not naturally apply to nonlinear neural networks. A naive solution to optimize multitask learning performance is to search through all possible combinations of tasks to select the best subset. This is costly to compute for a large number of tasks. Prior work \cite{standley2020tasks} computes pairwise task affinity scores by training one model for every pair of tasks and approximates higher-order task affinities by averaging first-order affinity scores. \citet{fifty2021efficiently,yu2020gradient} compute first-order task affinities using gradient similarity. These methods are much more efficient than naive search, but they ignore higher-order relationships among tasks, such as combining more than two tasks. This is critical as higher-order task affinities can exhibit more complexities. We observe that they cannot be reliably predicted from first-order affinity scores (cf. Figure \ref{fig_transfer_pred}, Section \ref{sec_transfer_prediction}) as transfer relationships are not monotone or submodular, meaning that adding more task data does not necessarily help performance (see Figure \ref{fig_transfer_relations}, Section \ref{sec_taskrel}). Moreover, naively computing all pairwise affinities requires fitting $T^2$ models given $T$ tasks, which is costly even for tens of tasks.

This paper's main contribution is to design an efficient algorithm to cluster tasks into similar groups while accounting for higher-order transfer relationships. One can compare the performance of a multitask model, such as a GNN trained on all tasks, against several multitask GNNs, each trained for a task group.
Section~\ref{sec_experiments} shows that this approach yields the best results among a wide set of baselines on various real-world datasets. Our method can be viewed as a \emph{boosting} procedure~\cite{breiman2001random} and can be used on top of any graph learning algorithm.

\medskip
\noindent\textbf{Approach.} We outline the overall procedure; See also Figure \ref{fig_pipeline} for an illustration.
Given $T$ tasks, we first compute a {\em task affinity score} $\theta_{i,j}$ for every pair of tasks $i$ and $j$.
A higher value of $\theta_{i, j}$ indicates task $j$ transfers better to task $i$ while also accounting for the presence of other tasks.
{Conceptually, $\theta_{i,j}$ is similar to the feature importance score in random forests when hundreds of other features are available.}
This {\em higher-order} task affinity score can also predict whether a set of tasks transfer positively or negatively to a target task.
Given the affinity score matrix $[\theta_{i, j}]_{T\times T}$, we use a spectral clustering algorithm \cite{ng2001spectral,shi2000normalized} to separate tasks into similar groups, which is more suitable for joint training. 
Specifically, our algorithm optimizes the sum of task affinity scores within groups through spectral clustering. 

Next, we describe the steps to estimate the affinity score matrix. A naive approach is to compute each entry individually, requiring $O(T^2)$ complexity.
We design an efficient sampling procedure that only requires $O(T)$ complexity.
Concretely, we sample $n = O(T)$ random subsets from $\set{1, 2, \dots, T}$ of a fixed size $\alpha$ (in practice, $\alpha = 5$ suffices).
We fit an MTL model for each subset and evaluate its prediction loss for each task in the subset;
Let $f_i(S)$ denote the prediction loss of task $i$, given a subset $S \subseteq\set{1, 2, \dots, T}$, which we evaluate on a holdout set.
Thus, $f_i(S)$ measures the information transfer from $S$ to $i$.
Then, we compute $\theta_{i,j}$ as the average among all subsets including $i,j$:
\begin{align} \theta_{i, j} = \frac{1}{n_{i,j}} \Big(\sum_{1\le k\le n:\, \set{i, j}\subseteq S_k} f_i(S_k)\Big), ~~\text{ for all } 1\le i,  j\le T. \label{eq_aff} \end{align}
To rigorously justify the rationale behind our affinity scores, we conduct a theoretical analysis in a stochastic block model style setting, where tasks follow a well-separated structure.
We prove that under this planted model, the affinity score matrix $[\theta_{i, j}]_{T\times T}$  exhibits a block-diagonal structure, each block corresponding to one cluster.
This characterization shows that the spectral clustering algorithm recovers the underlying task structure.

\medskip
\noindent\textbf{Summary of Contributions.}
The contribution of this work is threefold. 
First, we design a task affinity score to measure higher-order task relationships and estimate the scores with an efficient sampling procedure, which only requires fitting $O(T)$ MTL models.
Second, we propose a spectral clustering step to find task groups based on the affinity score matrix.
We provide recovery guarantees for this clustering procedure and show that the affinity scores can be used to provably group related tasks in a planted model.
Third, we validate the benefit of our boosting approach using various community detection and molecular graph prediction datasets. The experimental results show that our approach improves test accuracy over different community detection and MTL baselines.

\medskip
\noindent\textbf{Organization.}
The rest of this paper is organized as follows. We first outline related work in Section~\ref{sec_related}. In Section~\ref{sec_observation},  we provide empirical grounding for the claim that accounting for negative transfer among tasks is crucial for MTL on graphs. Our boosting procedure is described in Section~\ref{sec_method}, followed by a thorough empirical study of its performance in Section~\ref{sec_experiments}. Finally, in Section~\ref{sec_theory}, we describe the theoretical analysis of our algorithm.
In Appendix \ref{app_proof}, we provide complete proofs for our theoretical results.
In Appendix \ref{sec_omitted_results}, we describe additional experimental results left from the main text.
\section{Related Work}
\label{sec_related}

\subsection{Modeling task relationships in MTL}
The importance of learning from a pool of disparate data sources is well-recognized in the data mining literature \cite{ben2002theoretical}.
However, naively combining several heterogeneous data sources can result in negative interference between their feature representations \cite{wu2020understanding}.
Researchers have designed methods to extract shared information from different tasks.
For instance, explicit regularization applied to the representations of all tasks can help encourage information transfer \cite{evgeniou2004regularized}.
These regularization schemes can be rigorously justified for convex hypothesis classes \cite{nie2018calibrated}.
For nonconvex hypothesis spaces such as graph neural networks, explicitly regularizing the feature spaces of all tasks is a non-trivial challenge \cite{yang2016deep,bai2022saliency}.
\citet{ma2018modeling} introduce a mixture-of-experts model to capture the task relationships, with each expert being an MTL network.
\citet{yu2020gradient} design gradient-based similarity measures that can be efficiently computed using the cosine similarity of gradients during training.
This can be extended to measure the similarity between two sets of tasks by averaging the gradient of tasks in each set.
Recent work \cite{li2022task,li2023identification} points out that first-order affinity measures deteriorate as a transferability measure when applied to a large set of tasks.

\medskip
\noindent\textbf{Task Grouping.}
Instead of sharing layers and model parameters across all tasks, \citet{kumar2012learning} proposes mitigating negative transfers by dividing tasks into several related groups.
Our paper takes inspiration from Datamodels \cite{ilyas2022datamodels}, which extrapolates the outcome of deep networks as influence functions.
In particular, \citet{ilyas2022datamodels} find that a linear regression method can accurately approximate the outcome of deep nets trained with a subset of samples on popular image benchmarks.
Our results (e.g., Figure \ref{fig_transfer_pred}) show that the affinity scores can also accurately predict transfer types in multitask learning.

\subsection{Transferable graph neural networks}

Graph neural networks have emerged as widely used tools for graph learning.
Ideally, we want to learn a powerful embedding for all downstream tasks \cite{lee2017transfer,verma2019learning,gong2019graphonomy,qiu2020gcc,goel2014connectivity,zhang2019pruning}.
\citet{zhu2021transfer} analyzes the transferability of GNN from graph $A$ to graph $B$ and highlights the complex correspondence between structural similarity and transferability between GNNs.
Besides GNN, researchers have also observed negative interference while applying graph embedding to perform transfer learning on Graphs \cite{gritsenko2022graph}.
\citet{ju2023generalization,ju2022robust} show non-vacuous generalization bounds for graph neural networks in the fine-tuning setting using Hessian.
Our paper expands on these prior works in two aspects.
First, we consider a multi-label learning setting involving as many as 1000 tasks, whereas the work of \citet{zhu2021transfer,gritsenko2022graph} focuses on transfer learning from graph $A$ to graph $B$.
Second, we consider multiple node prediction tasks on a single graph, which is different from graph pretraining (e.g., \citet{qiu2020gcc,hu2019strategies}) and graph algorithmic reasoning \cite{velivckovic2022clrs}.

\medskip
\noindent\textbf{Multitask Learning Applications for Graph-Structured Data.}
Combining multiple graph learning tasks jointly can potentially enhance the performance of single tasks.
Our results support this claim in the context of supervised overlapping community detection.
Besides, we believe many graph learning tasks can be cast into the multitask learning framework.
For instance, consider extracting entity relationships on knowledge graphs; Each entity relation may be viewed as one task.
\citet{wang2019multi} find that learning the dependencies of different relations through multitask representation learning can substantially improve the prediction performance.
There has also been some study on the trade-off between fairness and accuracy in MTL \cite{wang2021understanding}.
It is conceivable that the new tools we have developed may benefit these related applications.
This is a promising direction for future work.

\subsection{Overlapping community detection}

Identifying community structures is one of the most widely studied problems in network science \cite{fortunato2010community}.
A common approach to finding communities given a seed set is to measure the local connectivity of a subgraph using spectral graph properties (e.g., the conductance of a cut).
\citet{yang2013overlapping} describe an efficient algorithm using non-negative matrix factorization for finding overlapping communities.
\citet{whang2013overlapping} finds local clusters by identifying low conductance set near a seed.
These approaches use the connectivity of edges to compute spectral properties.
Besides, higher-order structures from hypergraphs are found to be useful for overlapping community detection \cite{benson2016higher,yin2017local}.
Lastly, community detection can be cast in the correlation clustering framework, which does not require specifying the number of communities \cite{bonchi2022correlation, demaine2006correlation,ailon2008aggregating,mandaglio2021correlation}.

Our paper is closely related to the work of \citet{chen2017supervised}.
The difference is that we formulate the problem of predicting community labeling via multitask learning, whereas \citet{chen2017supervised} consider a multi-class classification setup.
Our formulation is more suitable when we are dealing with a large number of overlapping communities.
This is a novel perspective on community detection to the best of our knowledge.
Our results, compared with strong baselines including VERSE embedding \cite{tsitsulin2018verse} and BigClam \cite{yang2013overlapping}, suggest that modeling higher-order task relationships can significantly improve empirical performance for multitask learning.

\section{Investigating Task Relationships}\label{sec_observation}

We investigate task relationships in the setting of overlapping community detection.
We demonstrate that negative transfer is widespread across tasks and persists in large models.
We show that task relationships are neither monotone nor submodular in the higher-order regime.
Motivated by these considerations, we propose a task grouping problem for conducting MTL on graphs.

\subsection{Setup and background}

\begin{figure*}[t!]
    \begin{minipage}[b]{0.23\textwidth}
        \centering
        \includegraphics[width=0.9\textwidth]{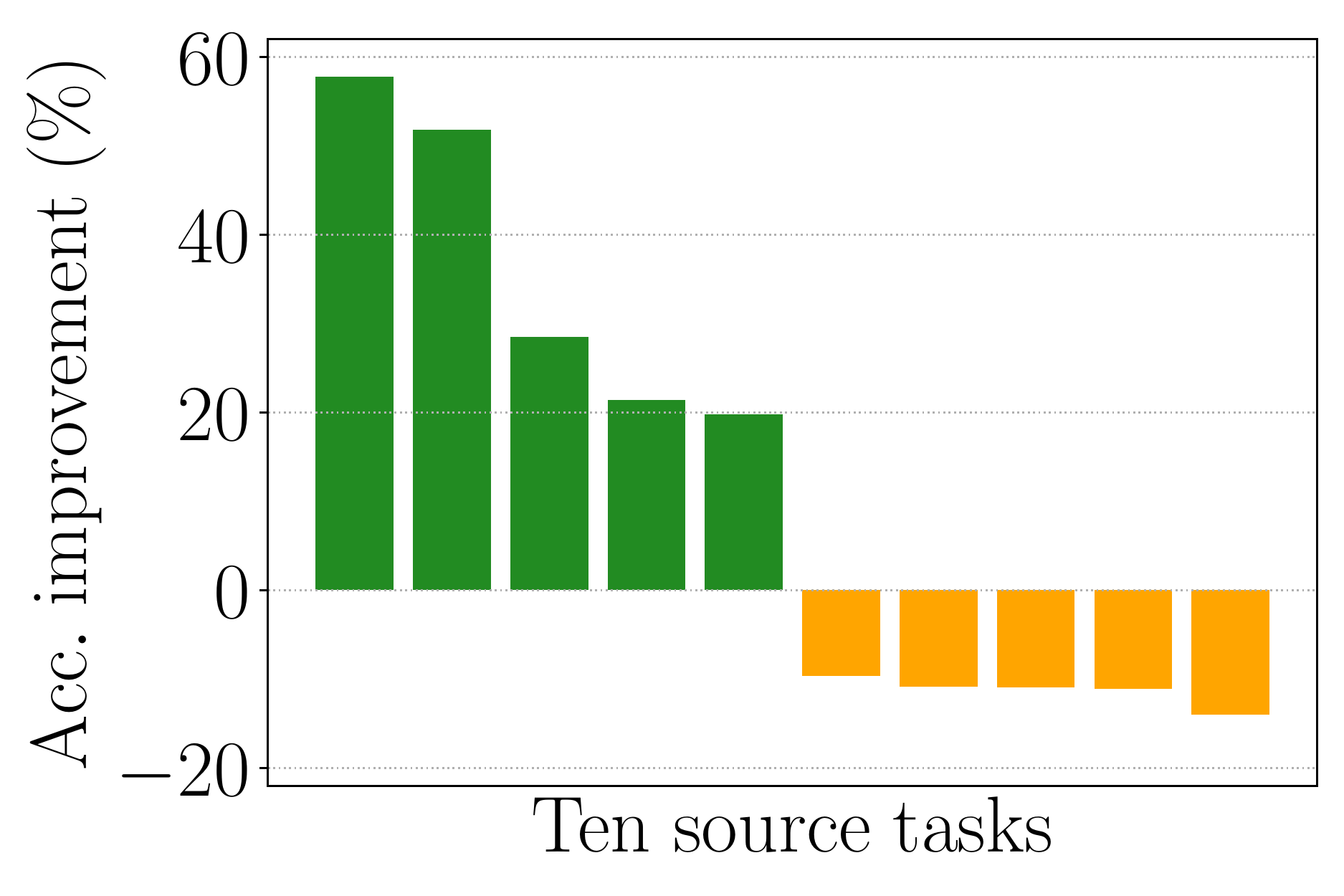}
    \end{minipage}\hfill
    \begin{minipage}[b]{0.23\textwidth}
        \centering
        \includegraphics[width=0.9\textwidth]{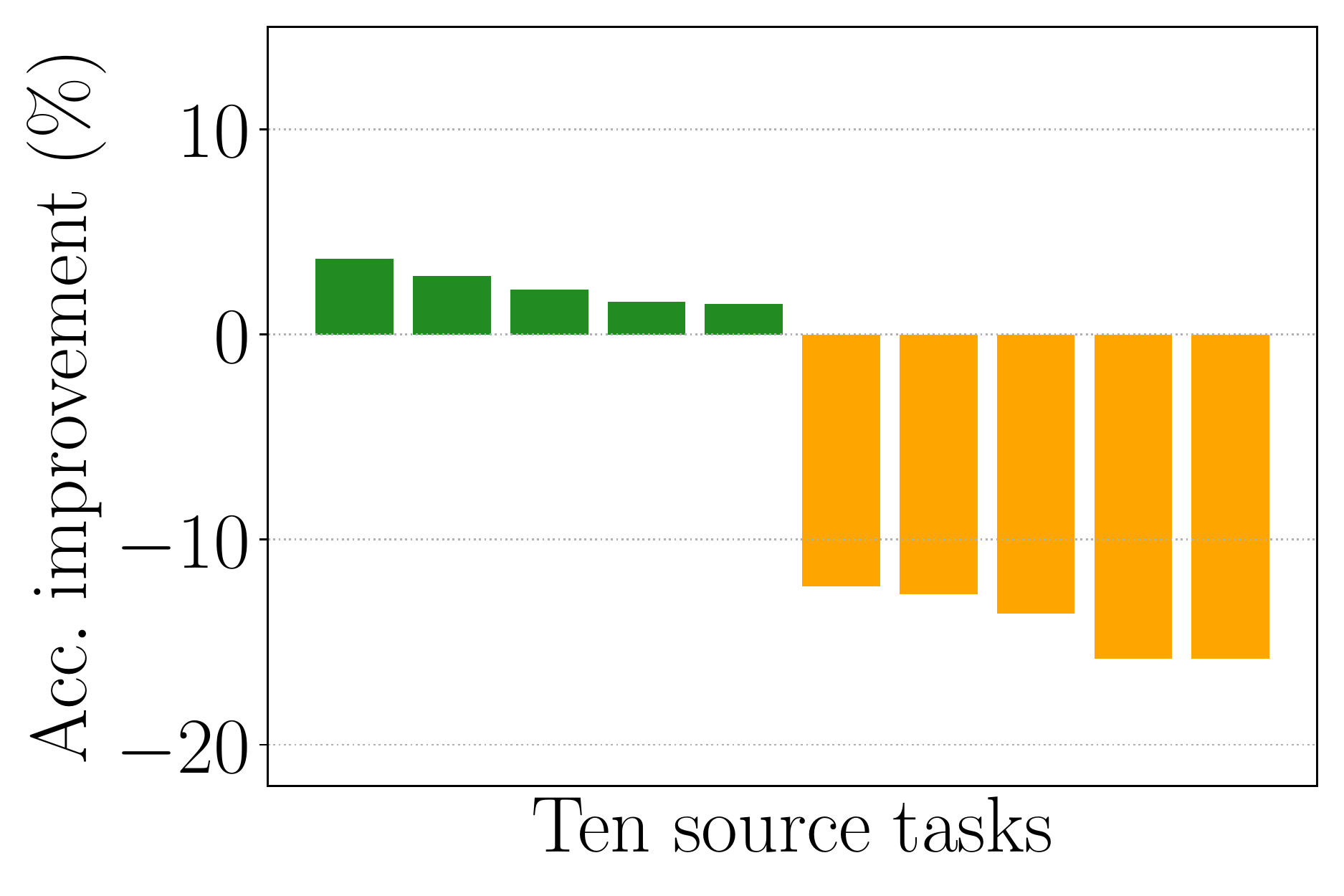}
    \end{minipage}\hfill
    \begin{minipage}[b]{0.23\textwidth}
        \centering
        \includegraphics[width=0.9\textwidth]{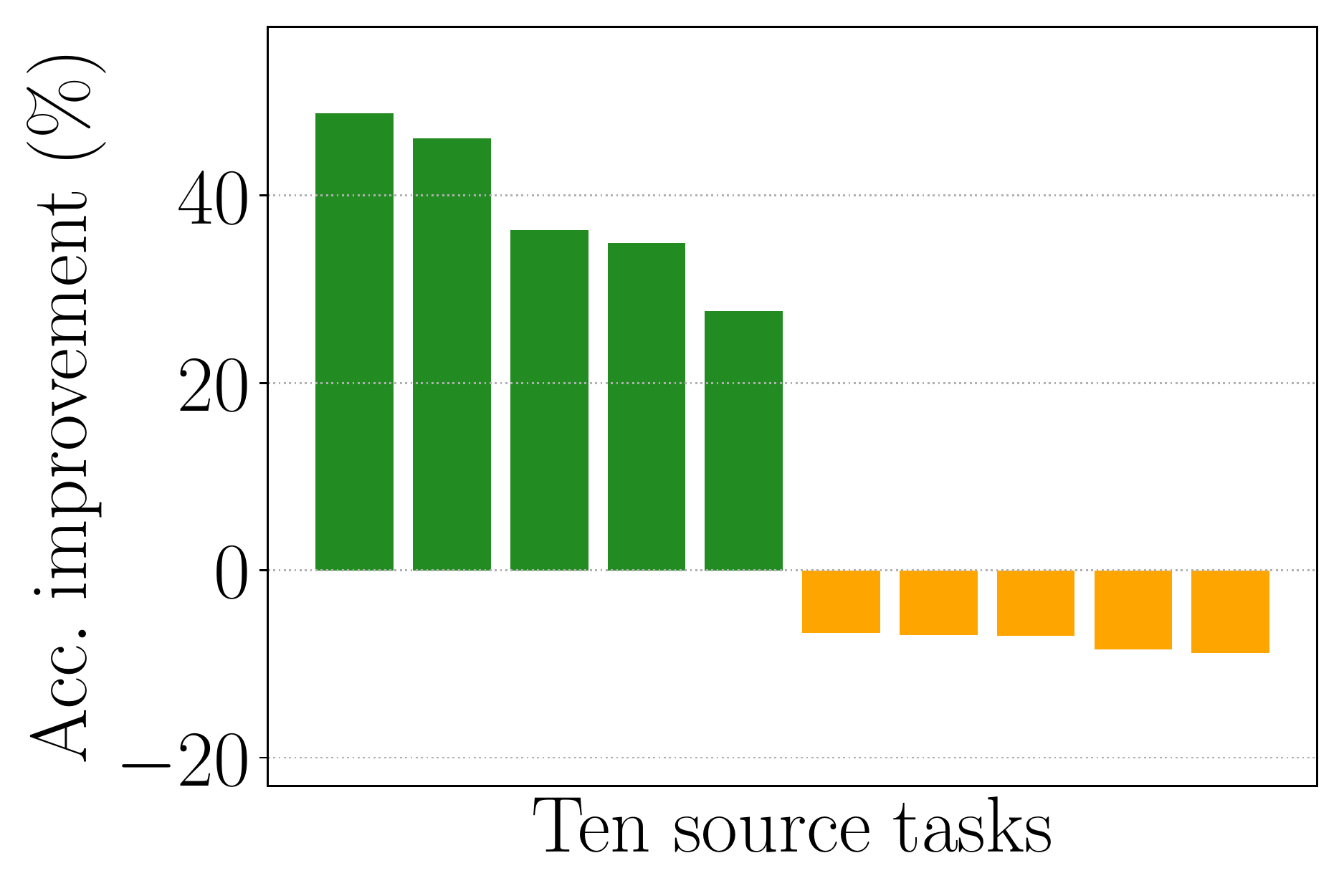}
    \end{minipage}\hfill
    \begin{minipage}[b]{0.23\textwidth}
        \centering
        \includegraphics[width=0.9\textwidth]{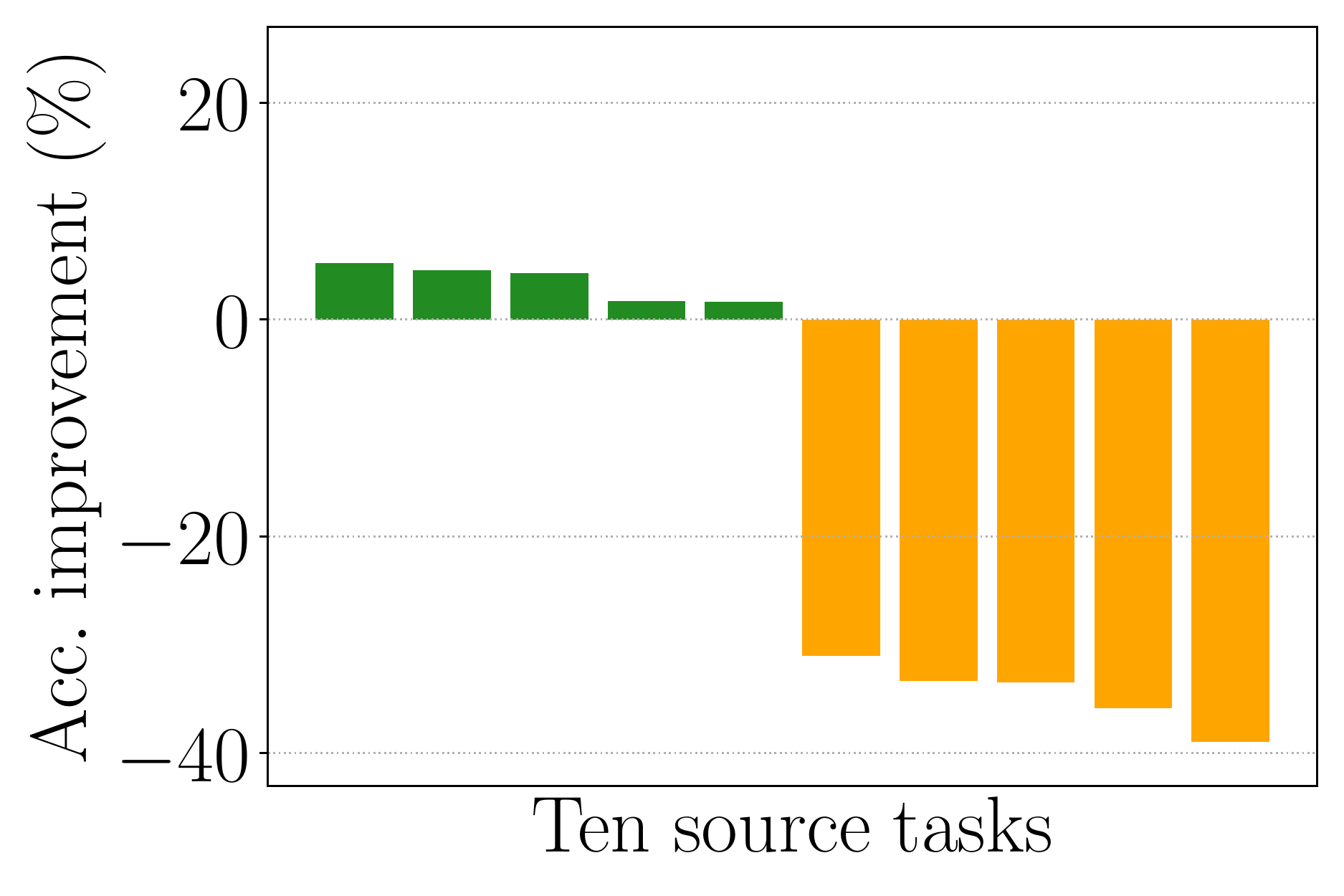}
    \end{minipage}
    \caption{This figure illustrates the widespread negative transfer effect among tasks by noting that MTL performance can dip below STL for four separate (randomly selected) target tasks. We fix a target task $i$ for each plot, then randomly pick ten source tasks (out of $100$) and for each source task $j$ train an MTL with $i$ and $j$; we report the MTL accuracy for $i$ minus $i$'s STL accuracy. Thus, bars above zero indicate positive transfers from source to target tasks, while bars below zero indicate negative transfers.}
    \label{fig_pairwise_transfer}
\end{figure*}

\begin{figure}[b!]
    \centering
     \begin{minipage}[b]{0.23\textwidth}
        \centering
        \includegraphics[width=0.99\textwidth]{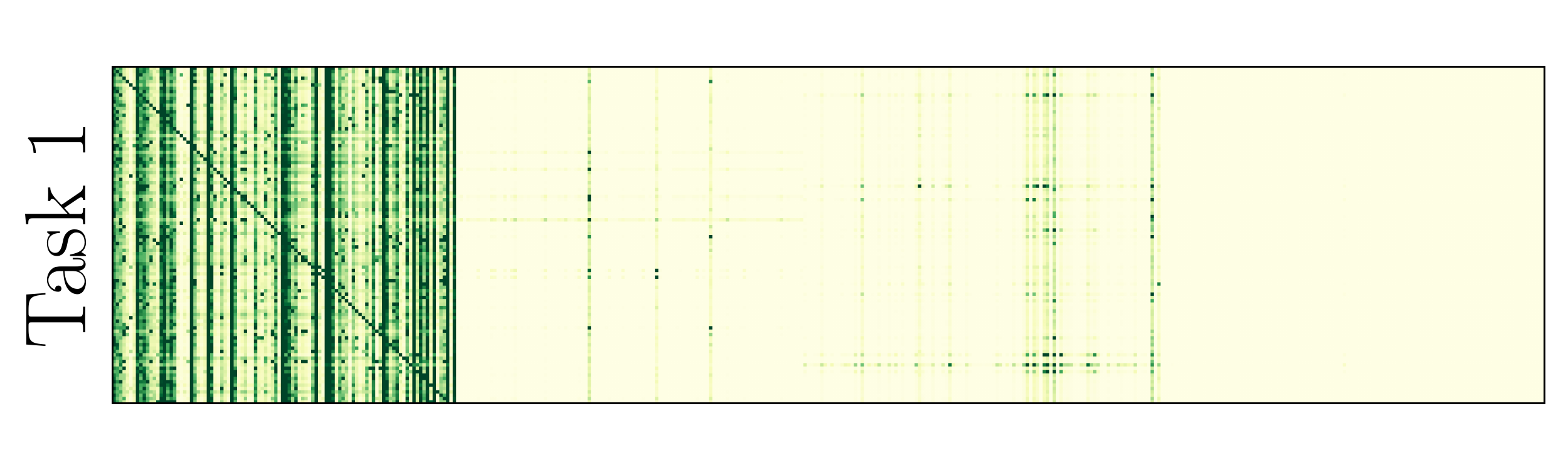}
    \end{minipage}
     \begin{minipage}[b]{0.23\textwidth}
        \centering
        \includegraphics[width=0.99\textwidth]{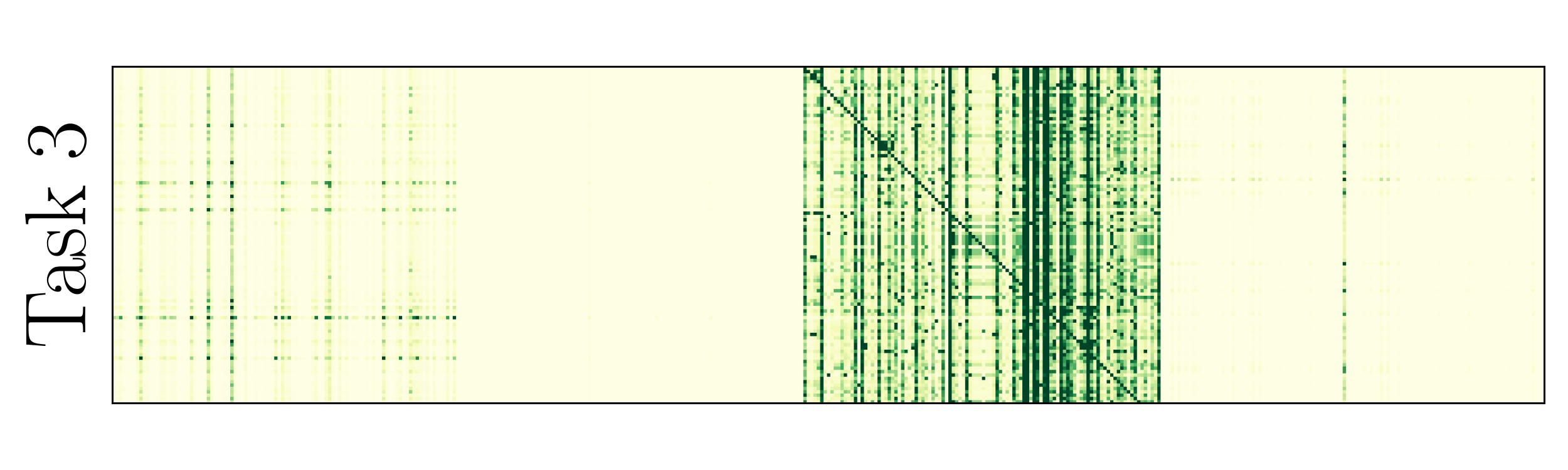}
    \end{minipage}
     \begin{minipage}[b]{0.23\textwidth}
        \centering
        \includegraphics[width=0.99\textwidth]{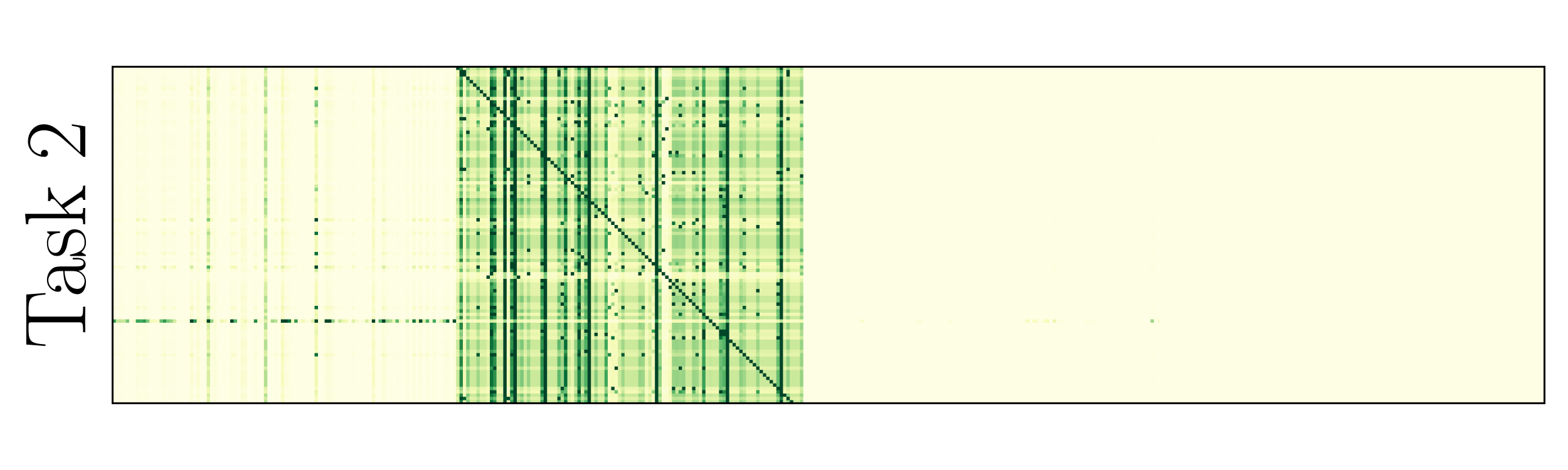}
    \end{minipage}
    \begin{minipage}[b]{0.23\textwidth}
        \centering
        \includegraphics[width=0.99\textwidth]{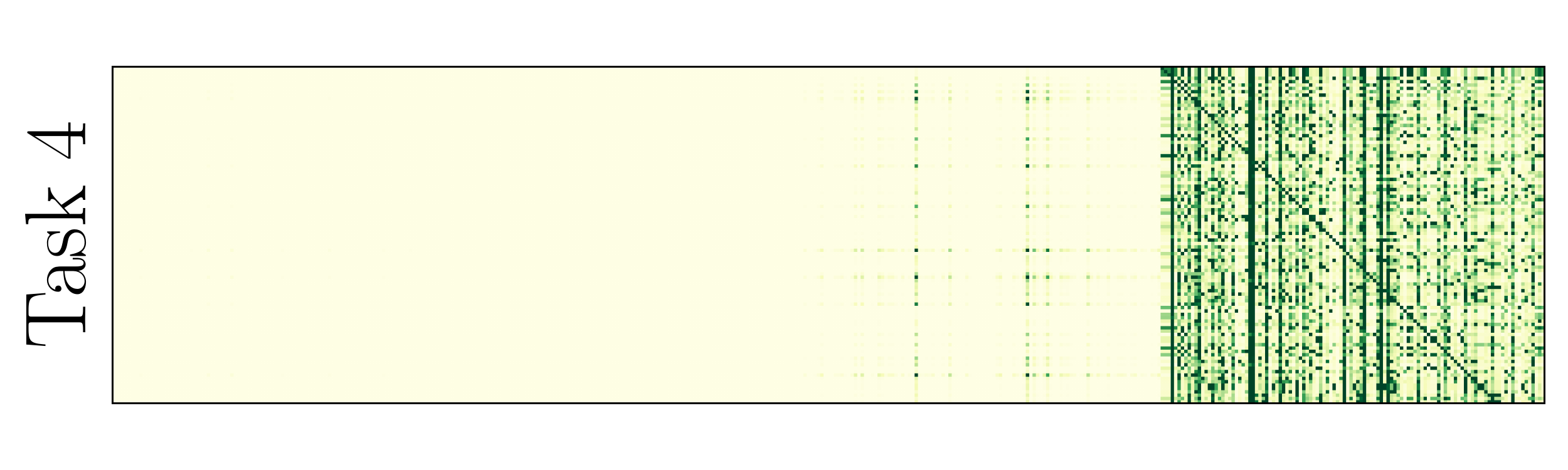}
    \end{minipage}
    \vspace{-0.1in}
    \caption{In each subfigure, we visualize the personalized PageRank vectors of a set of nodes in one community. They differ dramatically across non-overlapping communities.}
    \label{fig_propagation}
\end{figure}

\noindent\textbf{Problem setup.} We conduct an empirical study using multiple overlapping community detection tasks as a concrete example.
Given a graph $G = (V, E)$, we have a list of $T$ communities as subgraphs of $G$.
Let $C_1, C_2, \dots, C_T$ denote the vertex set of these communities.
We are given a vertex subset of each community during training as seed sets.
For every $i = 1, 2, \dots, T$, deciding whether a node $v \in V$ belongs to $C_i$ is a binary classification task.
Note that this formulation differs from supervised community detection \cite{chen2017supervised} but is more suitable for overlapping community detection.
This formulation is an example of multi-label learning, a particular case of multi-task learning (see, e.g., Fig. 1b of \citet{zhang2021survey}).
Casting multi-label learning into this more general formulation provides new perspectives in solving the problem.
\begin{itemize}[leftmargin=*]
    \item The prediction of membership in $C_i$ is task $i$, which is a binary classification task, given $G$ and node features.
    \item There are $T$ tasks, one for each community.
\end{itemize}

\smallskip
\noindent\textbf{Datasets.}
We use social network datasets with known community labels, including the Amazon, YouTube, DBLP, and LiveJournal networks from SNAP \citep{yang2012defining}. We use the 100 largest communities from each network and keep the subgraph induced by the nodes in the 100 communities. 
For each community detection task, we randomly sample 10\% of the nodes from the community in the training set, together with 10\% of nodes outside the community as negative samples. We randomly sample another 20\% of nodes as the validation set and treat the rest as a test set. We evaluate the performance of each task using the F1 score on the test set. See Table \ref{tab_community_detection} for details statistics of the four datasets.

\medskip
\noindent\textbf{Models.}
We consider an MTL model that consists of a single encoder to obtain shared representations and a separate prediction layer for each task.
We train this model by minimizing the average loss over all tasks' training data. 
Our experiments in this section use the SIGN model \citep{frasca2020sign} as the encoder, which is more efficient to train than GCN.
The encoder involves 3 layers, each with a fixed width of 256 neurons. Our choice of this encoder is without loss of generality, and our observations also apply to other encoders. 
We construct the node features from the VERSE embedding \cite{tsitsulin2018verse}, which encodes personalized PageRank vectors known as valuable features for community detection \cite{andersen2006communities}.

\medskip
\noindent\textbf{Negative transfer on graphs.} 
A common phenomenon with multitask learning is negative transfer \citep{pan2010survey}, meaning that combining one task with another worsens performance compared with training a task separately. 
We show that negative transfer occurs during MTL on graphs.
We take $100$ tasks from the YouTube dataset. First, we fix a randomly chosen task $i$ as the target task and use the rest as source tasks. Then, we train a GNN for task $i$, and 99 MTL models, each combining one source task with task $i$. The performance gap between the MTL and STL models indicates the transferability from a source task to task $i$.

Figure \ref{fig_pairwise_transfer} above shows the results of this experiment, repeated over four randomly chosen target tasks. The bars above zero correspond to \emph{positive transfers} as MTL performance exceeds STL, while bars below zero correspond to \emph{negative transfers}. We observe that both positive and negative transfers appear in all four settings.

\medskip
\noindent\textbf{Structural differences.} Why do negative transfers happen during multitask learning on graphs?
A common belief in the MTL community is that this is due to differences in the task labels \cite{wu2020understanding,yang2020analysis,li2023identification}.
We argue that graph neural networks involve another kind of heterogeneity due to graph diffusion. 
We appeal to a connection between GNN propagation and personalized PageRank (PPR) \cite{klicpera2018predict,bojchevski2020scaling,chen2020scalable}, positing that dramatically different PPR structures among communities will induce different graph diffusion for GNNs.
In Figure \ref{fig_propagation}, we visualize the PPR vectors of four randomly chosen tasks from the YouTube dataset.
Within each subfigure, each row corresponds to the PPR vector of one node that belongs to a particular community. We plot the PPR vectors of a set of nodes from the same community.
Clearly, PPR vectors differ dramatically for nodes from different communities, suggesting that the diffusion processes are highly heterogeneous.
We also observe that tasks yield positive transfers tend to have higher similarity between their PPR vectors. Detailed results are described in Appendix \ref{sec_ppr_similarity}.

\medskip
\noindent\textbf{Will larger models address negative transfers?} %
A natural approach to addressing negative transfers is to increase the model size, but this does not account for the above structural differences. We hypothesize that due to innate data heterogeneity, negative transfers between tasks cannot be addressed by increasing the model capacity. 
To verify this, we use the first target task from Figure \ref{fig_pairwise_transfer} and select the source task in the rightmost bar with the strongest negative transfer.
We gradually increase the number of neurons in the hidden layer {from 32, 64, 128, 256, 512, 1024, to 2048}, corresponding to larger model capacities.
Figure \ref{fig_model_capacity} shows the results.
We observe consistent negative transfers, i.e., the accuracy improvements stay negative, after increasing the model capacity. 
We have also observed the same result on a more powerful GNN with attention, i.e., GAMLP \citep{zhang2022graph}. See Appendix \ref{sec_model_size_GAMLP} for these results.

\begin{figure*}[t!]
    \begin{subfigure}[b]{0.33\textwidth}
        \centering
        \includegraphics[width=0.70\textwidth]{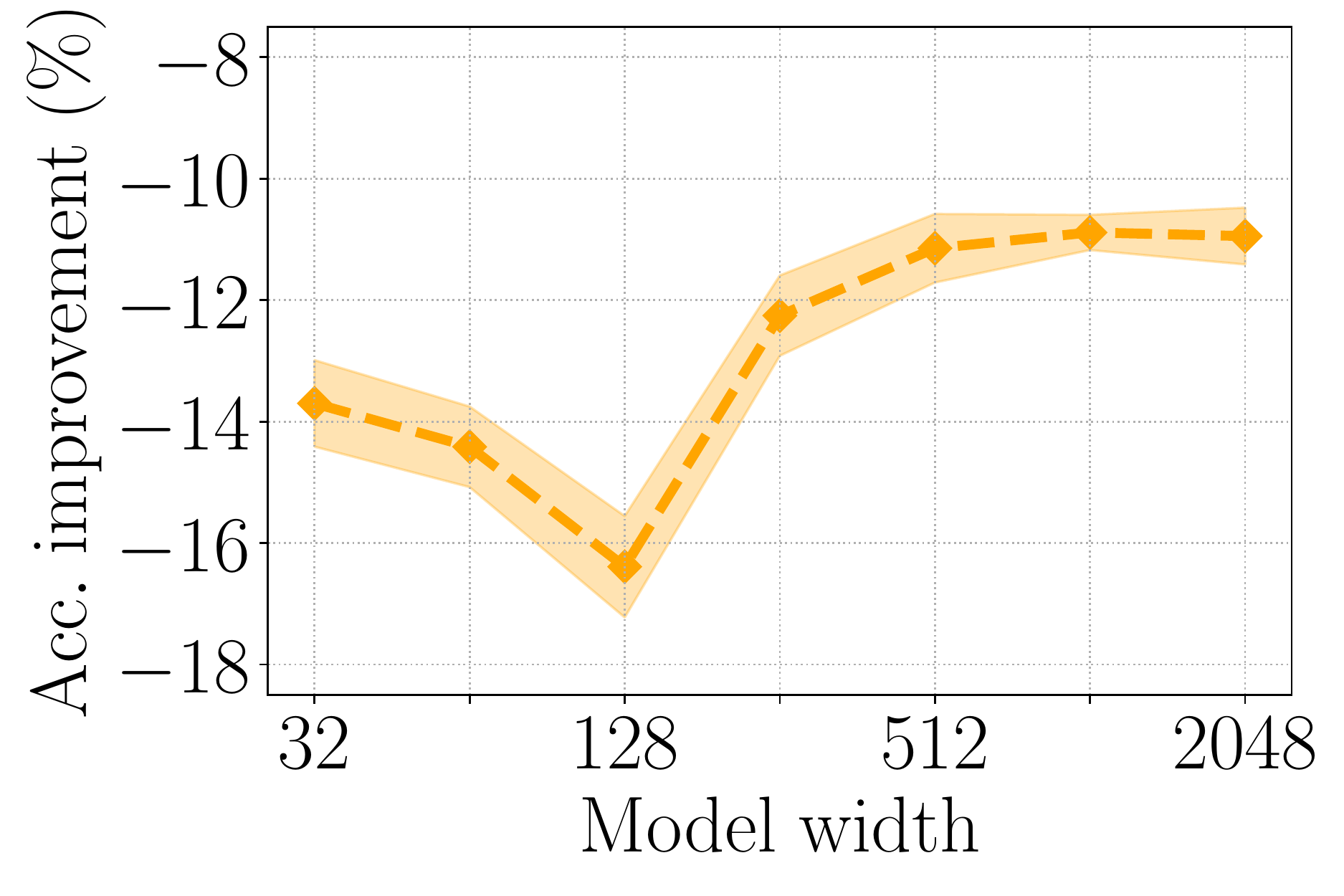}
        \vspace{-0.05in}
        \caption{Varying model size}\label{fig_model_capacity}
    \end{subfigure}
    \begin{subfigure}[b]{0.33\textwidth}
        \centering
        \includegraphics[width=0.7\textwidth]{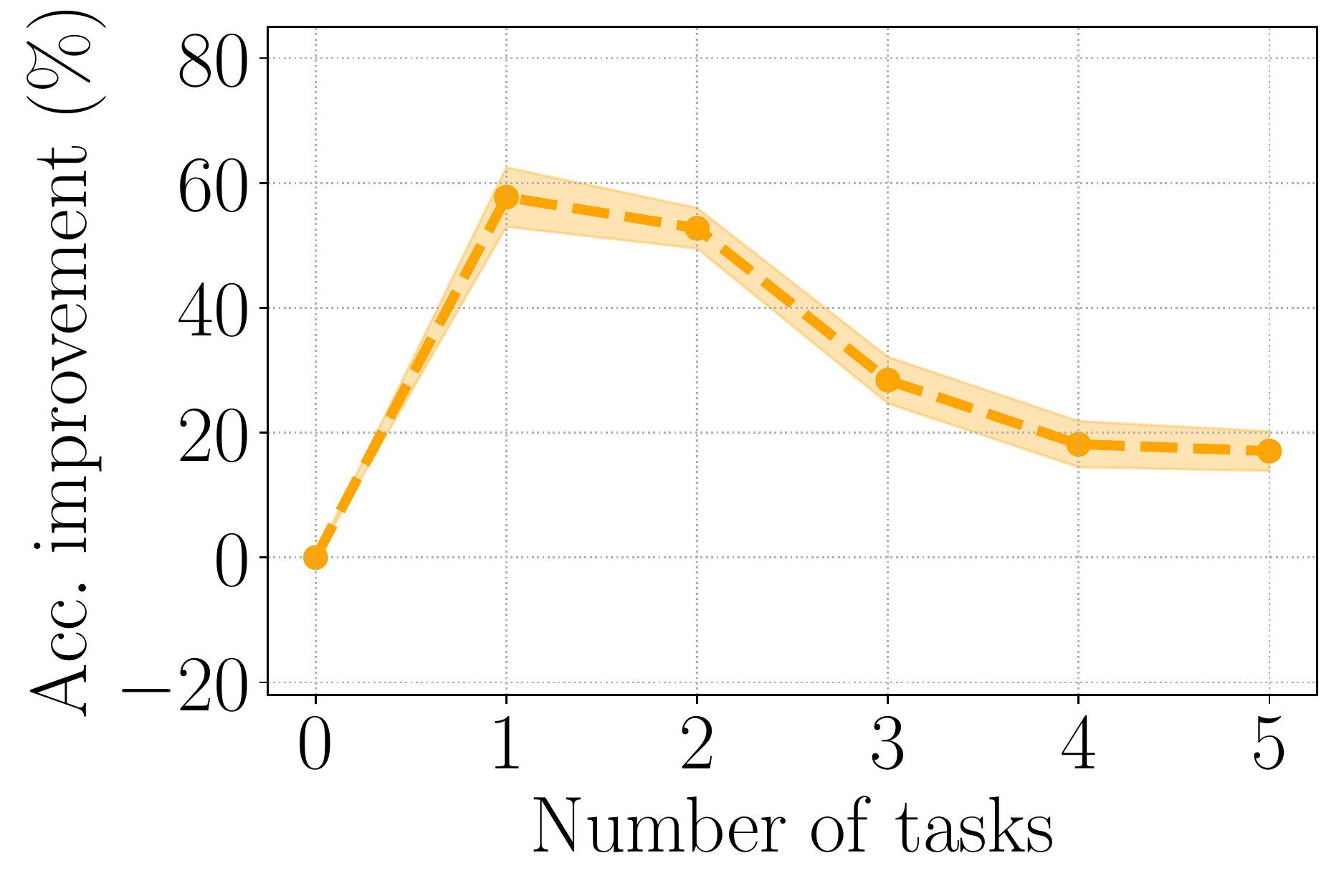}
        \vspace{-0.05in}
        \caption{$f$ is not monotone}\label{fig_mon}
    \end{subfigure}
    \begin{subfigure}[b]{0.33\textwidth}
        \centering
        \includegraphics[width=0.7\textwidth]{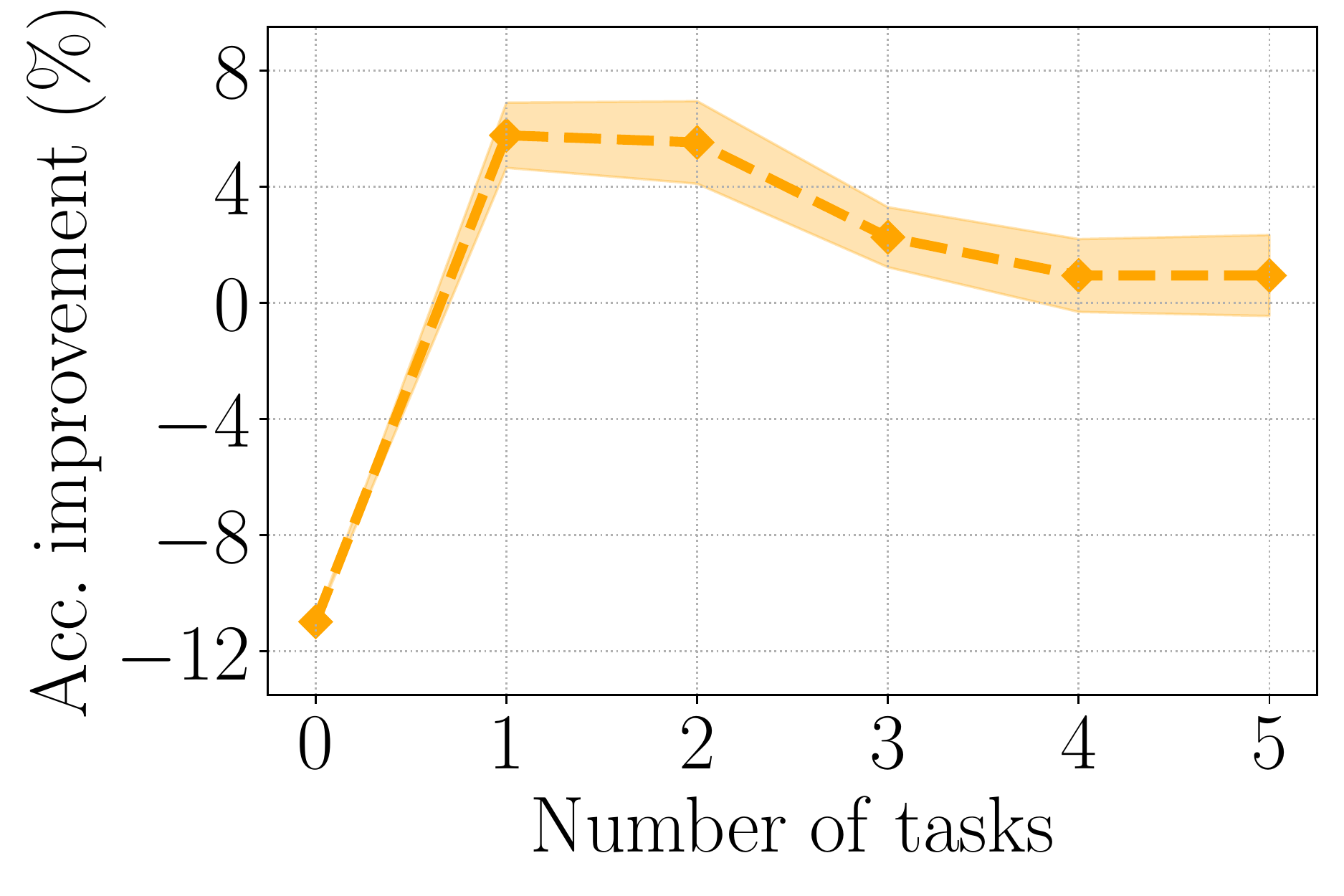}
        \vspace{-0.05in}
        \caption{$f$ is not submodular}\label{fig_sub}
    \end{subfigure}
    \vspace{-0.2in}
    \caption{
    (\ref{fig_model_capacity}) We show a consistent negative transfer even after increasing the model size (measured by width).
    (\ref{fig_mon}) The $x$-axis refers to the number of added source tasks to train with the target task. The $y$-axis refers to the difference in performance between MTL and STL (with the target task alone). We observe that the MTL performance of a target task starts to decrease after adding two or more source tasks, even though these are all ``positive'' source tasks (in the sense of pairwise transfer).
    (\ref{fig_sub}) Under the presence of a negatively interfering source task, the benefit of adding more ``positive'' tasks diminishes, implying that the $f(\cdot)$ function is not submodular.}
    \label{fig_transfer_relations}
\end{figure*}

\subsection{How do task relationships behave?}\label{sec_taskrel}

Next, we study the multitask learning performance involving more than two tasks.
At the extreme, this would involve $2^T$ combinations of task subsets.
To be precise, given any subset of tasks $S \subseteq \set{1, \ldots, T}$, let $f_i(S)$ denote the MTL performance of combining data from all tasks in $S$, evaluated on task $i$, for each $i \in S$.

\smallskip
\noindent\textbf{Q1: Is $f$ monotone?}
To better understand how $f$ behaves, we pick a target task $t$ and measure $f_t(\set{t})$.
Then, we add new tasks to be combined with $t$.
We only add a task $i$, if task $i$ is beneficial for task $t$, i.e., $f_t(\set{t, i}) \ge f_t(\set{t})$. %
Figure \ref{fig_mon} shows the result of applying the above setting to the first target task.
We observe that after adding more than two positive source tasks, the MTL performance decreases. 
This shows that $f_t(\cdot)$ is not monotone.

\medskip
\noindent\textbf{Q2: Is $f$ submodular?}
A function $f(\cdot)$ is submodular if for any two subsets $S \subseteq S' \subseteq\set{1, 2, \dots, T}$ and any single task $x$, $f(\{x\} \cup S') - f(S') \le f(\{ x \} \cup S) - f(S)$. 
We pick one negative source task as $x$ and ask if adding positive source tasks mitigates the negative transfer.
In Figure \ref{fig_sub}, we find that adding positive tasks does not always help, which implies that $f$ is not submodular.
The scales of $f$ are also different compared to Figure \ref{fig_mon}, because the presence of the negative task reduces the effect of positive tasks.
The takeaway is that $f$ is not monotone or submodular, motivating our approach to extrapolate $f$ via sampling.

\subsection{Task grouping for multitask graph learning}

We aim to obtain a set of networks where each network is trained on a subset of tasks. The objective is to optimize the overall performance of all tasks after combining the networks. 
To approach this problem, we consider finding subsets of tasks, namely task groups, to maximize the positive transfers of tasks within each group. %
We want to divide $\set{1, 2, \dots, T}$ into possibly overlapping subsets of tasks.
Let $\cS$ denote the collection of subsets.
Given $\cS$, the performance of task $i$ is the highest one among all networks:
\[ \cL_i(\cS) = \max_{X \in \cS} f_i(X). \]
Thus, the overall performance of all tasks on a solution is $\sum_{i=1}^{T} \cL_i(\cS)$. 
Suppose there is a limited inference budget $b$, the number of MTL models we can deploy in inference.
We want to find at most $b$ groups that achieve the highest performance: $\sum_{i=1}^T \cL_i(\mathbf{S})$.

To address this problem, we need to evaluate the multitask learning performance for all subsets of tasks, which total $2T$ combinations. 
Because task relationships are highly non-linear, we need a more efficient procedure to capture transfer relationships.
More generally, finding the optimal task groups is NP-hard via a reduction from the set-cover problem (see, e.g., \cite{standley2020tasks}). %

\section{Our Approach}\label{sec_method}

We now present our approach to optimize multitask model performance through grouping tasks that utilize higher-order task affinities. Recall the steps in our pipeline from Figure~\ref{fig_pipeline}:
(1,2) Repeatedly sample a random subset of tasks, and evaluate the model's performance that combines the tasks in each subset.
(3) Average the multitask learning performances over subsets that involve two specific tasks, yielding the task affinity score. 
(4) Then, we use these task affinity scores to group tasks using a spectral clustering algorithm on the matrix of task affinity scores.

\subsection{Estimating higher-order task affinities}

\noindent\textbf{Notations.} 
Suppose we are given a graph $G = (V, E)$ where $|V| = N$ and $|E| = M$, with node features $X \in \real^{N\times d}$. 
There are $T$ semi-supervised tasks on the graph. For each task $i$, we are given a set of nodes $\hat{V}^{(i)}$ with known labels ${Y}^{(i)}$. The goal of task $i$ is to predict the labels for the rest of the nodes on the graph $V/\hat{V}^{(i)}$. Note that the set $\hat{V}^{(1)}, \ldots, \hat{V}^{(T)}$ can be either overlapped or disjoint with each other. The objective is to optimize the average prediction performance over the $T$ tasks.

Let $\phi$ be the encoder network shared by all tasks. Let $\psi_1, \ldots, \psi_T$ be prediction layers for tasks $1, \ldots, T$ that map feature vectors to task outputs. When the input is a graph, we consider using graph neural networks as the encoder $\phi$.
Given $S\subseteq \set{1,2,\dots,T}$, let $\phi^{(S)}$ and $\psi^{(S)}_i$, for $i \in S$, be the trained model on the combined dataset of $S$.
We evaluate the prediction loss on each task's validation dataset.
Let $\widetilde{V}^{(i)}$ denote a set of nodes in $V$, which is used as a validation set for the task $i$.
Let $\ell$ be an evaluation metric, e.g., the cross-entropy loss.
We define multitask learning performance for any $i\in S$ as:
\begin{align}
    f_i(S) = \frac{1}{|\widetilde{V}^{(i)}|} \sum_{v \in \widetilde{V}^{(i)}} \ell\Big(\psi_i^{(S)}\big(\phi^{(S)}(X_v; G)\big), Y^{(i)}_{v}\Big) \label{eq_ft}
\end{align}%

\smallskip
\noindent\textbf{Measuring task affinity scores.}
Our approach measures higher-order task affinities to indicate how well a task transfers another task when combined in a subset. 
We show that such task affinity measures can be estimated by training $n$ models, where $n$ only needs to grow linearly to the number of tasks $T$. 
Moreover, our measure gives a more accurate prediction of higher-order multitask transfer results than previous notions of task affinity. 

We view task affinity as a transferability measure from source to target tasks.
Given a task $i \in \set{1, \ldots, T}$ as a target task, 
denote the affinity of another task $j$ to $i$ as  $\theta_{i,j}$. 
To model the relations of higher-order transfers. we define the task affinity score $\theta_{i,j}$ as the average MTL prediction loss $f_i(S)$ on target task $i$ over subsets that contain both task $i$ and $j$. 
We emphasize that the affinity scores account for the presence of other tasks.
Also note that a higher value of $\theta_{i,j}$ indicates higher usefulness of task $j$ to task $i$.

We estimate the task affinity scores through a sampling approach. Conceptually, this is analogous to graph embedding methods that optimize embeddings to approximate proximity scores. Similarly, we sample random subsets from tasks $1$ to $T$ and estimate the task affinity scores on the sampled subsets using this procedure:

\begin{enumerate}%
    \item Sample $n$ subsets from tasks $1$ to $T$, denoted as $S_1, \ldots, S_n$. We sample each subset from the uniform distribution over subsets with size $\alpha$.
    In other words, among all subsets of $\set{1, 2, \dots, T}$ with size $\alpha$ (note there are $\binom{T}{\alpha}$ of them), we pick one uniformly at random, with probability $1 / \binom{T}{\alpha}$.
    \item Evaluate prediction loss $f_i(S)$ for every task $i \in S_k$ and every subset $S \in \set{S_1, \ldots, S_n}$ by training a multitask model on $S$. 
    \item Estimate the task affinity scores $\theta_{i,j}$ by averaging the MTL performances $f_i$ over subsets containing both task $i$ and $j$:
    \begin{align}\label{eq_fit_scores}
        \theta_{i,j} = \frac{1}{n_{i, j}} \Big(\sum_{1 \leq k \leq n: \set{i, j} \subseteq S_k} f_i(S_k)\Big).
    \end{align}
    where $n_{i, j}$ is the number of subsets that contain tasks $i, j$.
    In particular, when $i$ and $j$ are the same, we set $\theta_{i, i}$ as the average of $f_i(S)$ over all $S$ having $i$.
\end{enumerate}
To summarize, the above procedure yields a $T$ by $T$ affinity score matrix, denoted as $\bm{\Theta} = [\theta_{i, j}]_{T\times T}$.

\subsection{Finding task groups by spectral clustering}

Since the affinity scores serve as a proxy of higher-order task relationships, we optimize MTL performance by finding task groups with the highest task affinity scores within each group. Our task grouping algorithm, as described below, contains two major steps.
The complete procedure is given in Algorithm \ref{alg_task_grouping}. 

\begin{algorithm}[h!]
\caption{Task Grouping Using Higher-Order Task Affinities}\label{alg_task_grouping}
\raggedright
\textbf{Input}: $T$ tasks; Training and validation sets of the task.\\
\textbf{Require}: The size of each subset $\alpha$; Number of sampled subsets $n$; Inference budget $b$; Multitask learning algorithm $f$. \\
\textbf{Output}: $b$ trained multitask models. \\
\begin{algorithmic}[1] %
    \STATE For $k = 1, \dots, n$, sample a random subset $S_k$ from $\set{1,2,\dots,T}$ with size $\alpha$; evaluate $\bm{f}(S_k)$ following equation \eqref{eq_ft}.
    \STATE Estimate the task affinity scores ${\bm{\Theta}}$ following equation \eqref{eq_fit_scores}.
    \STATE Generate task groups $\mathbf{S}^{\star} = \set{S_1, \ldots, S_b}$ by applying spectral clustering on a symmetric matrix constructed from $\bm{\Theta}$.
    \STATE Train $b$ multitask models for each task group $S_1, \ldots, S_b$.
\end{algorithmic}
\end{algorithm}

First, we construct a transformed affinity score matrix for clustering. 
Since the sum of affinity scores between two tasks $i$ and $j$ within a group is $(\theta_{i,j} + \theta_{j,i})$, we define a symmetric matrix $\bm{A}_1 = (\bm{\Theta} + \bm{\Theta}^\top)/2$. 
Additionally, we find auxiliary source tasks for each group that yield positive transfer to the group. This is achieved by viewing the matrix $\bm{\Theta}$ as directional task relationships, with source tasks represented along the rows and target tasks along the columns.
To find a set of source tasks that yield the highest affinity scores to a set of target tasks, it is natural to consider the symmetrized matrix $ \left[\begin{smallmatrix}
 \bm{0} & \bm{\Theta} \\
 \bm{\Theta}^{\top} & \bm{0}
\end{smallmatrix}\right]$.
Thus, based on the affinity score matrix $\Theta$, we construct a symmetric matrix: 
$$\bm{A} = \left[\begin{matrix}
 \bm{A}_1 & \bm{\Theta} \\
 \bm{\Theta}^{\top} & \bm{0}
\end{matrix}\right].$$%

Second, we apply spectral clustering algorithms (e.g., \citet{ng2001spectral,shi2000normalized}) on $\bm{A}$ and merge the clustered target and source tasks in one group of final task groupings. 
Afterward, we train one multitask model for each group by combining all the data from that group.

{}
\section{Experiments}\label{sec_experiments}

We now evaluate our approach empirically on various community detection and molecular graph data sets.
First, we show that our task affinity scores can be estimated efficiently to predict negative transfers more accurately than first-order task affinities. 
Second, we apply our approach to overlapping community detection tasks on several datasets with ground-truth community labels: our approach outperforms the naive MTL by \textbf{3.98\%} and task grouping baselines by \textbf{2.18\%}.
Third, we evaluate our approach on molecular graph prediction tasks and show a \textbf{4.6\%} improvement over prior MTL methods. 
Lastly, we provide ablation studies to show that our approach is stable under various settings. {{The code for reproducing our experiments is available at \url{https://github.com/VirtuosoResearch/boosting-multitask-learning-on-graphs}.}}

\subsection{Results for predicting negative transfers}\label{sec_transfer_prediction}

\noindent\textbf{Experiment setup.} 
We use task affinity scores $\theta_{i, j}$ for predicting negative transfers as follows. Given a target task $i$ and a subset of tasks $S$ containing $i$, we predict whether the subset $S$ transfers negatively to a task $i$, i.e., the MTL prediction loss $f_i(S)$ of training task $i$ with subset $S$ is worse than the STL loss of task $i$.

We set up the prediction as binary classification.
For each task $i$, input feature for a subset $S$ is the task affinity scores of tasks in $S$ to task $i$: $\mathbbm{1}_S \circ [\theta_{i, 1}, \ldots, \theta_{i, T}]$.
The label is whether the subset $S$ transfers negatively to a task $i$. 
Then, we fit a logistic regression model that maps the features to the binary labels. 
We fit $T$ models in total and evaluate the average F1-score over the $T$ models. 

We evaluate the above prediction on the YouTube dataset with $T = 100$ tasks and estimate task affinities of order $5$, $10$, and $20$ (which means that the size of sampled subsets is $\alpha = 5$, $10$, or $20$). We use the
transfer results on $n=2000$ task subsets to fit logistic regression models and evaluate the predictions on 500 new task subsets that do not appear in training.

\begin{figure*}[t!]
    \begin{minipage}[b]{0.99\textwidth}
        \centering
        \includegraphics[width=0.9\textwidth]{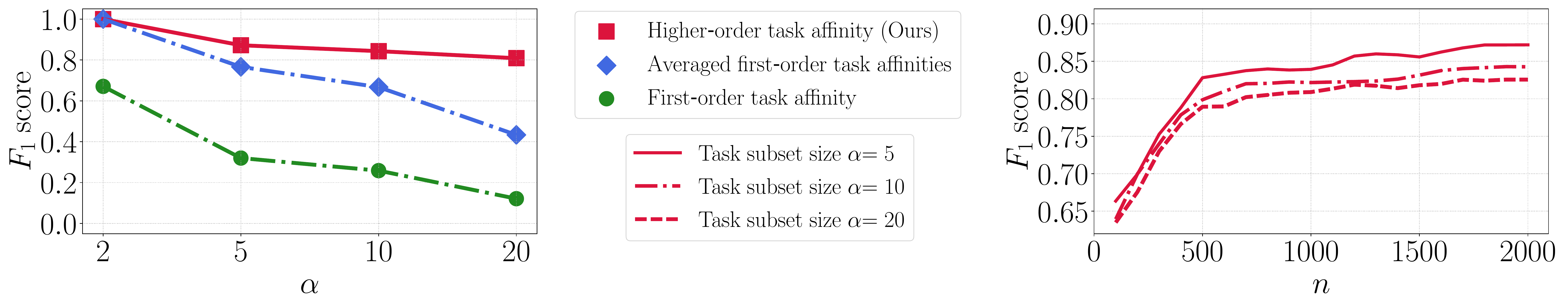}
    \end{minipage}
    \vspace{-0.1in}
    \caption{
    We use task affinity scores from tasks in a subset $S$ to task $i$ to predict whether training with subset $S$ decreases the STL performance of task $i$. 
    Left:  Compared with two first-order task affinity scores, our higher-order task affinity scores achieve consistently better F1-score for predicting negative transfers of combining up to $\alpha = 20$.
    Right:  The F1-score for predicting negative transfers converges when the sampled subsets $n$ reach $2000$. Results consistently hold for different subset sizes.}
    \label{fig_transfer_pred}
\end{figure*}

\smallskip
\noindent\textbf{Results.} 
First, we illustrate the convergence of  
$F_1$-score of negative transfer prediction when increasing the sample size $n$, as shown in the right of Figure \ref{fig_transfer_pred}.  
We observe that with $n \leq 2000 = 20T$, using higher-order task affinity scores predicts negative transfers with $F_1$-score above 80\%.
This result consistently holds for sampling subsets of different sizes. 

Second, we compare our approach to two previous notions of affinity scores. 
One involves computing first-order task affinity through the effect of one task's gradient on another task's loss \cite{fifty2021efficiently}. Another approximates the higher-order task transfers by averaging the first-order task affinity scores \cite{standley2020tasks}. 
Figure \ref{fig_transfer_pred} on the left shows that the $F_1$-score from previous measures gradually gets worse; Ours are accurate for subsets of size ranging from $2$ to $20$.

\begin{figure}[t!]
    \begin{minipage}[b]{0.4\textwidth}
        \centering
        \includegraphics[width=0.8\textwidth]{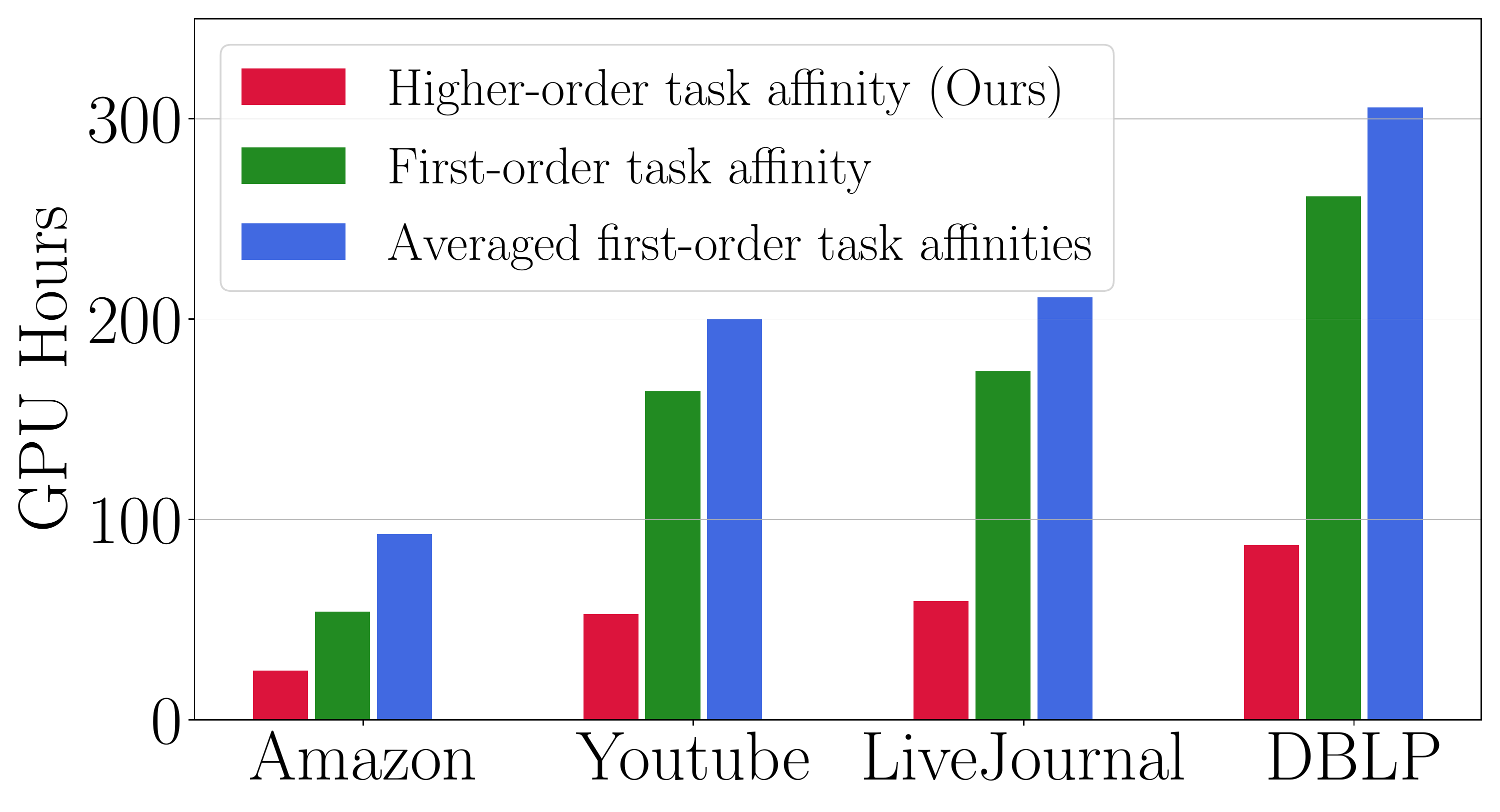}
    \end{minipage}
    \vspace{-0.1in}
    \caption{Comparing the runtime for computing higher-order vs. first-order task affinity \cite{standley2020tasks,fifty2021efficiently}.}
    \label{fig_runtime}
\end{figure}

\smallskip
\noindent\textbf{Run time analysis.}  Then, we present the run time of our approach. 
Our approach requires training $n$ networks, one for each random subset. 
We find that using $n \leq 20T$ samples suffice for estimating the task affinity scores to convergence.  
In contrast, previous methods \cite{standley2020tasks,fifty2021efficiently} estimate task affinities for every pair of tasks, resulting in training on $O(T^2)$ task pairs. 
Concretely, we report the run time of our approach in Figure \ref{fig_runtime}, evaluated on a single NVIDIA RTX GPU. 
Compared with the two previous methods, our approach requires \textbf{3.7}$\times$ less running time, averaged over the four data sets.

\medskip
\noindent\textbf{Speed up training.}
In practice, we can further speed up training with early stopping and downsampling.
Our experiments found that this leads to a significant speed-up, and the overhead on top of Naive MTL is only up to 2-3$\times$. In Table \ref{tab_running_time} of Appendix \ref{sec_running_time}, we report the running time of our approach by incorporating the speed-up methods, as compared with all MTL baselines.

\subsection{Results for overlapped community detection}\label{sec_clu_alg}

\noindent\textbf{Baselines.} 
We compare our approach with three classes of baseline approaches, which are selected to be representative in terms of relative improvement.
The first class of methods is a set of popular methods for community detection, including:
\begin{itemize}%
\item BigClam \cite{yang2013overlapping}.
\item Louvain clustering \cite{blondel2008fast}.
\item Network embedding methods including Node2Vec \cite{grover2016node2vec} and VERSE \cite{tsitsulin2018verse}. We use a logistic regression classifier on the node embedding for each community.
\item GNN-based community detection methods including MinCutPool \cite{bianchi2020spectral} and Deep Modularity Networks \cite{tsitsulin2023graph}.
\end{itemize}
Second, we consider two baseline methods that optimize all tasks using a shared model: 
\begin{itemize}%
\item Naive MTL \cite{caruana1997multitask} trains all tasks jointly in a shared model.
\item Mixture-of-Experts \cite{ma2018modeling} that trains multiple expert models on all tasks and uses a gating network to combine model outputs for each task in a weighted manner.
\end{itemize}
Third, we consider four task grouping baselines that find task groups and train a network on each group.
\begin{itemize}%
\item Forward selection: Start from all empty groups. Enumerate all tasks, add the task to each existing group, and assign it to the group, resulting in the best average performance. 
\item Backward selection: Start from a group with all tasks and other groups as empty. Enumerate all tasks in the first group, combine the task with the rest groups, and assign the task to the group resulting in the best average performance. 
\item First-order task affinity that evaluates the effect of one task's gradients on another task's loss \cite{fifty2021efficiently}.
\item Averaging the first-order task affinity scores to approximate higher-order task affinities \cite{standley2020tasks}. 
\end{itemize}
We note that \citet{fifty2021efficiently} and \citet{standley2020tasks} use a branch-and-bound algorithm to search for task groups but do not scale to one hundred tasks for the data sets.
We use their task affinity scores to compare these two methods but apply the spectral clustering procedure to find task groups.

\medskip
\noindent\textbf{Implementations.} 
We use a 3-layer SIGN model as the encoder for all MTL approaches, with a width of 256 neurons.
We use the VERSE embedding as the input node features. 
We compare approaches for splitting tasks into $b = 20$ groups. 
We use the same amount of model parameters for other MTL baselines.
In the evaluation, we report the macro F1-score of predicted communities to ground-truth community labels on the test set.

For our approach, we set the size of the subset as $\alpha=10$ and the number of samples as $n=2000$. 
We ablate the two parameters in Section \ref{sec_ablation} and find that the performance of our approach remains stable while varying them. 
We set the MTL performance metric $f_i(S)$ as the (negative) cross-entropy loss on the validation set.
We apply the spectral clustering algorithm in \cite{ng2001spectral,shi2000normalized} to find task groups on the symmetric adjacency matrix constructed from task affinity scores.

\begin{table*}[t!]
\begin{minipage}[t]{0.54\textwidth}
\centering
\caption{Macro F1-score of community detection tasks on four social networks. We compare our approach with graph embedding, MTL optimization, and feature subset selection methods. We report the averaged result for each experiment over three random seeds, including the standard deviations.}\label{tab_community_detection}
\begin{small}
\begin{tabular}{@{}lcccc@{}}
\toprule
Dataset   & {Amazon} & {Youtube} & {DBLP} & {LiveJournal}  \\ 
Nodes & {3,225} & {16,751} & {57,368} & {18,433}  \\
Edges & {20,524} & {104,513} & {420,122} & {1,397,580}  \\
\midrule
BigClam     & 27.30 $\pm$ 0.26  & 18.84 $\pm$ 0.18 & 13.46 $\pm$ 0.11  & 22.50 $\pm$ 0.31 \\
Louvain clustering    & 60.95 $\pm$ 0.19  & 29.03 $\pm$ 0.34 & 36.73 $\pm$ 0.34  & 64.08 $\pm$ 0.17  \\
Node2Vec    & 39.05 $\pm$ 0.10  & 32.44 $\pm$ 0.18 & 28.72 $\pm$ 0.10  & 50.40 $\pm$ 0.29 \\
VERSE       & 61.00 $\pm$ 0.32  & 38.17 $\pm$ 0.12 & 53.48 $\pm$ 0.24  & 58.71 $\pm$ 0.48 \\
Naive MTL   & 87.08 $\pm$ 4.04 &  43.42 $\pm$ 2.24  & 67.95 $\pm$ 2.28 & 82.56 $\pm$ 3.21 \\ 
Multi-Gate MoE  & 88.92 $\pm$ 6.65 & 44.65 $\pm$ 4.28 & 68.83 $\pm$ 4.06  &  83.08 $\pm$ 4.89\\
Forward Select  & 90.45 $\pm$ 3.63 & 47.62 $\pm$ 2.84 & 68.58 $\pm$ 2.96 &  86.19 $\pm$ 2.61  \\
Backward Select  & 90.41 $\pm$ 5.98 & 47.68 $\pm$ 3.01 & 68.63 $\pm$ 2.97 &  85.91 $\pm$ 4.18 \\
\midrule 
\textbf{Alg. \ref{alg_task_grouping} (Ours)} & \textbf{92.66 $\pm$ 4.85}  & \textbf{49.62 $\pm$ 2.26} & \textbf{70.68 $\pm$ 2.65} & \textbf{88.43 $\pm$ 2.70} \\
\bottomrule
\end{tabular}
\end{small}
\end{minipage}
\hfill
\begin{minipage}[t]{0.44\textwidth}
\centering
\caption{Test performance on multitask molecular graph prediction datasets. We compare our approach with MTL optimization methods and feature subset selection methods. We report the averaged result over three random seeds, including the standard deviations.}\label{tab_molecule}
\begin{small}
\begin{tabular}{@{}lcccc@{}}
\toprule
Dataset   & QM9 & Alchemy & OGB-molpcba  \\ 
Metric & MAE ($\downarrow$) & MAE ($\downarrow$) & AP ($\uparrow$) \\
Graphs & 129,433 & 202,579 & 437,929 \\
Nodes per graph & 18.0 & 10.1 & 26.0 \\
Edges per graph & 18.6 & 10.4 & 28.1 \\
Tasks & 12 & 12 & 128 \\
Groups & 3 & 3 & 20 \\
\midrule
Naive MTL   & 0.081 $\pm$ 0.003 & 0.103 $\pm$ 0.001 & 27.03 $\pm$ 0.23 \\ 
Multi-Gate MoE  & 0.079 $\pm$ 0.003 & 0.100 $\pm$ 0.001 & 28.38 $\pm$ 0.34  \\
Forward Select  & 0.077 $\pm$ 0.002 & 0.099 $\pm$ 0.001 & 28.72 $\pm$ 0.21  \\
Backward Select  & 0.073 $\pm$ 0.002 & 0.095 $\pm$ 0.004 & 28.50 $\pm$ 0.16 \\
\midrule 
\textbf{Alg. \ref{alg_task_grouping} (Ours)} & \textbf{0.067 $\pm$ 0.003} & \textbf{0.090 $\pm$ 0.001} & \textbf{29.73 $\pm$ 0.12} \\
\bottomrule
\end{tabular}
\end{small}
\end{minipage}
\end{table*}

\medskip
\noindent\textbf{Results.} 
Table \ref{tab_community_detection} reports the evaluation of four social networks with ground-truth community labels. 
First, we find that VERSE embedding achieves the best performance among all the node embedding methods. Thus, we use the VERSE embedding as a node feature for conducting MTL on graph neural networks. Due to the space constraint, we report the results of other community detection methods in Appendix \ref{sec_additional_results}. 

\begin{itemize}%
\item {\itshape Benefit of task grouping:} 
Compared with methods that optimize a joint model on all tasks, task grouping consistently performs better than the naive MTL and Mixture of Experts. 
Our approach outperforms them by \textbf{3.98\%} on average. 

\item {\itshape Benefit of modeling higher-order relationships:}
Compared with forward and backward selection, our approach achieves an average improvement of \textbf{2.18\%} over the datasets. %
Moreover, we compare our approach with clustering by two first-order task affinities. The results show that our approach outperforms them by \textbf{2.49\%} on average. 
This validates the advantage of using higher-order task affinities over first-order task affinities.
The results are shown in Table \ref{tab_first_order} of Appendix \ref{sec_additional_results}.  
\end{itemize}

\subsection{Results for molecular graph prediction}

Next, we apply our approach to molecular graph prediction tasks, including two multi-task regression data sets from TUDatasets \cite{morris2020tudataset} and one multi-task classification dataset from OGB \cite{hu2020open}.
In the graphs, nodes represent 3D coordinates of atoms in molecules, and edges encode distances between atoms. Each task corresponds to predicting a specific molecular property.
We use a 6-layer GIN model as the encoder, with a width of 64. 
We evaluate the mean absolute error (MAE) on the regression datasets and the average precision (AP) on the classification dataset.

Table \ref{tab_molecule} compares our approach with MTL baselines, including naive MTL, Mixture of Experts, and forward/backward selection.
We find that on these three data sets, our method still outperforms the baselines relatively by \textbf{4.6\%} on average.

\subsection{Ablation studies}\label{sec_ablation}

\noindent\textbf{Number of task groups $b$.} We discuss how the number of task groups is determined in our approach. We hypothesize that a larger number of task groups gives greater flexibility and tends to have better performance. Ideally, we can generate $T$ task groups, each for a particular target task, and select helpful source tasks for the target task in each group. We validate the hypothesis by varying the number of task groups between 5, 10, 20, and 100. The results validate that more group achieves better performance. Interestingly, we also find that using 20 groups achieves results comparable to those of using 100 groups. Thus, we set $b=20$ in our experiments. The details are reported in Appendix \ref{sec_abl_param}.

\medskip
\noindent\textbf{Subset size $\alpha$.} Recall that we collect MTL prediction losses through sampling random subsets of a size $\alpha$. We evaluate the performance of our approach by varying the size $\alpha \in \set{5, 10, 20}$. 
First, we observe similar convergence results using different sizes, as shown in Figure \ref{fig_transfer_pred}. 
Next, we apply algorithm \ref{alg_task_grouping} with different values of $\alpha$. We notice that the performances are comparable. Using $\alpha = 10$ achieves slightly better performance than the other two. We posit that using a larger $\alpha$ does not help because the number of related tasks in our community detection data sets is limited. 

\medskip
\noindent\textbf{Number of samples $n$.} We further explore how $n$ affects the algorithm results. Our observation in Figure \ref{fig_transfer_pred} is that collecting $n = 20T$ is sufficient for task affinity scores to converge. Meanwhile, using a smaller $n$ can also achieve near 80\% F1-score for predicting negative transfers.
Thus, we test the performance of algorithm \ref{alg_task_grouping} by varying $n \in \set{1000, 1500, 2000}$. 
We observe that using $n=1000$ still achieves comparable performance as using $n=2000$. The performance difference is within 0.5\%.

{}
\section{Theoretical Analysis}
\label{sec_theory}

In this section, we aim to develop a principled understanding of our higher-order affinity scores $[\theta_{i, j}]_{T\times T}$.
To this end, we study a planted model in a theoretical setup, where the tasks are assumed to follow a block structure.
We note that planted models have been widely used to study graph clustering \cite{abbe2017community}.
In this setting, we ask:
\begin{itemize}
    \item Do our affinity scores provably capture higher-order task relationships?
    \item Could the affinity scores be used to successfully separate related tasks from each block?
\end{itemize}
We provide a positive answer to both questions in a theoretical setup, where the labels of each task have been drawn from a linear model.
Our finding is that for two tasks from the same group in the planted model, their affinity scores will be provably higher than two tasks from different groups.
To describe this result, we first formally introduce the setup.

\medskip
\noindent\textbf{Setup.} Suppose we are learning $T$ tasks.
For each task, $i$ from $1$ up to $T$, let the node labels of this task be given by a vector $\tilde Y^{(i)}$, all of which are observed on a set of nodes denoted as $\tilde X$.
Let $m$ denote the size of the observed set of nodes.
We focus on regression tasks in the analysis.
Thus, the values of $\tilde Y^{(i)}$ are all real values.
To further simplify the analysis, we consider a one-layer linear graph diffusion layer as $f(X, G)= P_{_G} X$, where $P_{_G}$ (e.g., the normalized graph Laplacian matrix) denotes the diffusion matrix of the graph neural network, and $X$ denotes the matrix of node features.
We assume that $X$ is drawn from an isotropic Gaussian distribution, and $P_{_G}$ is full rank.
We measure the loss of this GNN against the label vector $\tilde Y^{(i)}$ using the Mean Squared Error (MSE):
\begin{align}
    \ell_{i}(W) = \frac 1 m \bignorms{\tilde P_{_G} \tilde X W - \tilde Y^{(i)}}^2. \label{eq_W}
\end{align}
where $\tilde P_{_G}$ denotes the propagation matrix restricted to the set of nodes in $\tilde X$.
Based on our algorithm, we first show that the relevance score $\theta_{i, j}$ admits a structure that measures the distance between $\tilde Y^{(i)}$ and $\tilde Y^{(j)}$.
When we sample a set of tasks $S \subseteq \set{1, 2, \dots, T}$ with a total of $\alpha$ tasks, then we average their per-task losses as
\begin{align}\label{eq_sample_loss}
    \ell_S(W) = \frac 1 {\alpha} \sum_{i \in S} \ell_i(W).
\end{align}

\medskip
\noindent\textbf{Notations.} We follow the convention of big-O notations for stating the result.
Given two functions $h(n)$ and $h'(n)$, we use $h(n) = \textup{O}(h'(n))$ or $h(n) \lesssim h'(n)$   to indicate that $h(n) \le C \cdot h'(n)$ for some fixed constant $C$ when $n$ is large enough.

\medskip
\noindent\textbf{Characterization of affinity scores.} Minimizing equation \eqref{eq_W} over $W$ leads to a closed form solution on $W$---let us denote this as $\hat W_{S}$, which we can then plug into task $i$'s loss $\ell_i(\hat W_{S})$.
We then average the value of $\ell_{i}(\hat W_{S})$ for subsets $S_1, S_2, \dots, S_n$ that include $j$ as part of the subset.
This gives the relevance score $\theta_{i, j}$:
\begin{align}
    \theta_{i, j} = \frac 1 {n_{i, j}} \sum_{1\le k \le n:\, \set{i, j} \subseteq S_k} \ell_i\big(\hat W_{S_k}\big).
\end{align}
In the following result, we derive an explicit form of the score $\theta_{i, j}$.
For any matrix $A \in \real^{m \times n}$, let $A^\dagger$ be its Moore-Penrose inverse.

\begin{lemma}\label{lemma_score}
    In the setting described above,
    let the projection matrix $\tilde \Sigma$ be given by $\tilde \Sigma = \tilde{P}_{_G} \tilde{X} \big( \tilde{X}^\top \tilde{P}_{_G}^\top \tilde{P}_{_G} \tilde{X} \big)^\dagger \tilde{X}^\top \tilde{P}_{_G}^\top$.
    For any $1 \le i, j \le T$, we have that the relevance score $\theta_{i, j}$ is equal to
    \begin{align}\label{eq_theta}
       \theta_{i, j} = \frac 1 {n_{i, j} \cdot m} \sum_{1 \le k \le n:\, \set{i, j} \subseteq S_k} \bignorms{ \tilde \Sigma \cdot \Bigg(\frac 1 {\alpha} \sum_{l \in S_k} \tilde Y^{(l)} \Bigg) - \tilde Y^{(i)}}^2. 
    \end{align}
\end{lemma}

\begin{proof}[Proof of Lemma \ref{lemma_score}]
    From equation \eqref{eq_sample_loss}, $\hat{W}_S$ is the quadratic minimization of $L_S(W)$, expressed as 
    \begin{align*}
        \hat{W}_S = \Big( \tilde{X}^\top \tilde{P}_{_G}^\top \tilde{P}_{_G} \tilde{X} \Big)^\dagger \tilde{X}^\top \tilde{P}_{_G}^\top \Bigg(\frac 1 {\alpha} \sum_{i \in S} \tilde{Y}^{(i)} \Bigg).
    \end{align*}
    We have that for any subset $S \subseteq \set{1, 2, \dots, k}$,
    \begin{align}\label{eq_loss_hat_w}
        \ell_i(\hat{W}_S) = \frac 1 m \bignorms{\tilde \Sigma \cdot \Bigg(\frac 1 {\alpha} \sum_{l \in S} \tilde{Y}^{(l)} \Bigg) - \tilde Y^{(i)}}^2.
    \end{align}
    and for any $1 \le i, j \le n$, by definition (cf. equation \eqref{eq_aff})
    \begin{align}
        \theta_{i,j} &= \frac 1 {n_{i, j} \cdot m} \sum_{1\leq k \leq n:\, \set{i, j} \subseteq S_k} \bignorms{ \tilde \Sigma \cdot \Bigg(\frac 1 {\alpha} \sum_{l \in S_k} \tilde Y^{(l)}\Bigg) - \tilde Y^{(i)}}^2
    \end{align}
    Hence, the proof of equation \eqref{eq_theta} is completed.
\end{proof}

\noindent\textbf{Block structure of affinity scores.} Next, we show that under a separation condition between the label vectors $Y^{(1)}, Y^{(2)}, \dots, Y^{(T)}$, our algorithm can be used to separate the tasks into separated groups provably.
More precisely, let $\Sigma = P_{_G} X (X^{\top} P_{_G}^{\top} P_{_G} X)^{\dagger} X^{\top} P_{_G}^{\top}$.

\begin{assumption}\label{assume}
Suppose that the label vectors can be separated into $C$ groups such that:
\begin{itemize}%
    \item For any two different $i$ and $j$ within each group,
        \[ \bignorms{\Sigma \Big(Y^{(i)} - Y^{(j)} \Big)} \le a. \]
    \item For any two different $i'$ and $j'$ from different groups,
        \[ \bignorms{\Sigma \Big(Y^{(i')} - Y^{(j')} \Big)} \ge b. \]
\end{itemize}
\end{assumption}

Now we are ready to state a structural characterization on the affinity scores $[\theta_{i,j}]_{T \times T}$.

\begin{theorem}\label{thm_separate}
    Suppose the features and the label vectors of every task are given based on the setup specified within this section and satisfy Assumption \ref{assume}.
    Assume that the propagation matrix $P_{_G}$ is full rank.
    Let $\epsilon$ be a fixed constant that does not grow with $n$ and $m$, and $\delta$ be a value that is less than one.
    
    When $m = O\big(d^2 \log^4 d \log^2(T \delta^{-1}) \varepsilon^{-2}\big)$,
    $n = O\big(\log(T \delta^{-1}) \varepsilon^{-2}\big)$,
    and $b^2 - a^2 \ge O\big(d^4 \log^8 d \log^4(T \delta^{-1}) \varepsilon^{-4}\big)$, then with probability at least $1 - \delta$ over the randomness of the training samples, the affinity scores $[\theta_{i, j}]_{T \times T}$ satisfy the following block structure:
    For any $1\le i, j, k \le T$ such that $i, j$ come from the same group, but $i, j'$ come from different groups, then
    \begin{align}
        \theta_{i, j'} - \theta_{i, j} \ge \frac{\varepsilon} 2.
    \end{align}
\end{theorem}

Our result characterizes the block structure of the affinity scores.
The affinity scores are lower within the block of nodes from the same community.
Conversely, the affinity scores would be higher for any pair of nodes that cross two communities.
Based on this characterization, it is clear that by applying a spectral clustering algorithm over $[\theta_{i,j}]_{T \times T}$, one could recover the group structures specified under Assumption \ref{assume}.
For future work, it would be interesting to strengthen our analysis with relaxed assumptions further and extend it to more
general losses and planted models.
The complete proof can be found in Appendix \ref{app_proof}.

\section{Conclusion}

This paper studied multitask learning on graphs and designed a generic boosting procedure that improves MTL by finding related tasks and training them in groups.
We first estimate higher-order task affinities by sampling random task subsets and evaluating multitask performances. 
Then, we find task groups by clustering task affinity scores.   
Experiments show that our higher-order task affinity scores predict negative transfers from multiple tasks to one task.
Our approach improves over previous MTL methods on various community detection data sets.
The theoretical analysis further demonstrates that using our task affinity scores provably separates related and unrelated tasks.

Our work opens up interesting questions for future work. For example, can we incorporate various node or link prediction tasks to enhance community detection in multitask learning? 
Can we apply recent developments in correlation clustering  \cite{mandaglio2021correlation,bonchi2022correlation} to determine the number of clusters?

\smallskip
\noindent\textbf{Acknowledgement.}
Thanks to the anonymous referees for providing constructive feedback that improved our work significantly.
Thanks to Ruoxuan Xiong, Xiaojie Mao, and Yi Liu for fruitful discussions during various stages of this work.
D. L. is supported by a start-up fund from Khoury College of Computer Sciences, Northeastern University.

\balance
\begin{refcontext}[sorting=nyt]
    \printbibliography
\end{refcontext}

\appendix
\clearpage
\onecolumn
\section{Proof of Theorem \ref{thm_separate}}\label{app_proof}

\noindent\textbf{Proof Sketch:}
    The proof of Theorem \ref{thm_separate} is by carefully examining the sampling structures of our algorithm.
    A key insight is that $\theta_{i, j}$ measures the relevance score of task $i$ to task $j$ while accounting for the presence of other tasks besides $i$ and $j$.
    There are two steps in the proof.
    First, we show that $\theta_{i,j}$ deviates from a population average that takes into account all subsets of size $\alpha$ by an error of $O(n^{-1/2})$.
    This is based on Hoeffding's inequality.
    See equation \eqref{eq_theta_dist} in Appendix \ref{app_proof}.
    Second, since $\tilde \Sigma$ is a projection matrix based on equation \eqref{eq_bar_theta}, we have that for any $i$ and $S_k$,
    \begin{align}
          \bignorms{\tilde \Sigma \cdot \Bigg( \frac 1 {\alpha} \sum_{l \in S_k} \tilde Y^{(l)} \Bigg) - \tilde Y^{(i)} }^2  
        = \bignorms{\tilde\Sigma \cdot \Bigg( \frac 1 {\alpha} \sum_{l \in S_k} \tilde Y^{(l)} - \tilde Y^{(i)} \Bigg)}^2
        + \bignorms{\Big(\id - \tilde\Sigma\Big) \cdot \tilde Y^{(i)}}^2. \label{eq_sep}
    \end{align}
    Thus, for $\theta_{i, j'} - \theta_{i, j}$, the second term from equation \eqref{eq_sep} is canceled out. Then, we apply Assumption \ref{assume} to separate the first term between $i, j$ and $i, j'$.
    Our proof follows several previous works that analyze information transfer in multitask learning \cite{yang2020analysis,li2023identification}.

\medskip
Next, we argue that for any $i, j$, the following holds with probability $1 - \delta$.

    \begin{claim}\label{claim_err1}
    In the setting of Theorem \ref{thm_separate}, we have that for any $1 \le i < j \le T$:
    \begin{align}\label{eq_ydist}
        \bignorms{\frac 1 m \tilde X^{\top} \tilde P_{_G}^{\top} \Big(\tilde Y^{(i)} - \tilde Y^{(j)}\Big) - \frac{1}{N} X^{\top} P_{_G}^{\top}\Big(Y^{(i)} - Y^{(j)}\Big)} \leq \frac{4 B \sqrt d \log d \bignorms{P_{_G}}  \log ( 2T^2 \delta^{-1} ) }{\sqrt m }.
    \end{align}
    \end{claim}

    \begin{proof}
    We use the matrix Bernstein inequality (cf. Theorem 6.1.1 in Tropp (2015)) here. Let the $i$-th column of $X P_{_G}$ be denoted as $u_i$. 
    Thus, we have $u_i^\top \Big( Y^{(i)} - Y^{(j)} \Big) \leq 2BN\bignorms{P_{_G}} \sqrt d \log d$
    For all $t \ge 0$, we have:
    \begin{align*}
        \Pr\left[\bignorms{ \frac{1}{m} \tilde X^{\top} \tilde P_{_G}^{\top} \Big(\tilde Y^{(i)} - \tilde Y^{(j)}\Big) - \frac {1} {N} X^\top P_{_G}^\top \Big(Y^{(i)} - Y^{(j)}\Big)} \ge t \right]
        \le 2T^2 \cdot \exp\left( - \frac {(mt)^2 / 2} {(2B \norms{P_{_G}} \sqrt d \log d)^2 m + (2B\norms{P_{_G}} \sqrt d \log d) mt / 3}\right).
    \end{align*}
    \end{proof}
    Recall we have assumed that $X$ is drawn from an isotropic Gaussian.
    Thus, each row vector of $P_{_G} X$ follows a Gaussian distribution with a full-rank covariance matrix since $P_{_G}$ is full rank.
    Thus, by Gaussian concentration results, for a random subset of $m$ row vectors of $P_{_G} X$, they must be invertible.
    Thus, we will next argue that under the condition $m \ge d \log d$, with probability at least $1 - \delta$, the following holds.
    
    \begin{claim}\label{claim_err2}
    In the setting of Theorem \ref{thm_separate}, provided that $m \ge d \log d$, with probability at least $1 - \delta$, we have that
    \begin{align}
        \bignorms{\Bigg(\frac 1 m\tilde X^{\top} \tilde P_{_G}^{\top} \tilde P_{_G} \tilde X\Bigg)^{-1} - \Bigg(\frac 1 N \frac{}{} X^{\top} P_{_G}^{\top} P_{_G} X\Bigg)^{-1} }
        \lesssim \frac {2 \bignorms{P_{_G}} \sqrt d \log d \cdot \log (2T^2 \delta^{-1})} {\sqrt m}. \label{eq_error}
    \end{align}
    \end{claim}

    \begin{proof}
    We use the matrix Bernstein inequality (cf. Theorem 6.1.1 in Tropp (2015)) to deal with the above spectral norm.
    Let the $i$-th column of $X^{\top} P_{_G}^{\top}$ be denoted as $v_i$. 
    Thus, we have that $\norm{v_i} \le \bignorms{P_{_G}} \cdot \sqrt d \log d$.
    
    In expectation over the randomness of $\tilde P_{_G} \tilde X$,
    for all $t \ge 0$, we have:
    \begin{align*}
        \Pr\left[ \bignorms{\frac 1 m\tilde X^{\top} \tilde P_{_G}^{\top} \tilde P_{_G} \tilde X - \frac 1 N  X^{\top} P_{_G}^{\top} P_{_G} X } \ge t \right]
        \le 2T^2 \cdot \exp\left( - \frac {(mt)^2 / 2} {(\norms{P_{_G}} \sqrt d \log d)^2 m + (\norms{P_{_G}} \sqrt d \log d) mt / 3}\right).
    \end{align*}
    With some standard calculations,
    this implies that for any $\delta \ge 0$, with probability at least $1 - \delta$, equation \eqref{eq_error} holds.
    
    Next, denote by
    \[ E = \frac {\tilde X^{\top} \tilde P_{_G}^{\top} \tilde P_{_G} \tilde X} {m} - \frac{ X^{\top} P_{_G}^{\top} P_{_G} X} {N}
    ~~\text{ and }~~ A = \frac{X^{\top} P_{_G}^{\top} P_{_G} X} {N}. \]
    By the Sherman-Morrison formula calculating matrix inversions, we get
    \begin{align}
         \bignorms{\Big( \frac{\tilde X^{\top} \tilde P_{_G}^{\top} \tilde P_{_G} \tilde X}{m}  \Big)^{-1} - \Big( \frac{X^{\top} P_{_G}^{\top} P_{_G} X }{N}  \Big)^{-1}}
        = \bignorms{(E + A)^{-1} - A^{-1}} \nonumber 
        = &\bignorms{A^{-1} \Big( A E^{-1} + \id_{T\times T}\Big)^{-1}} \nonumber\\
        =& \bignorms{A^{-1} E \Big( A + E \Big)^{-1}} \nonumber \\
        \le& \big(\lambda_{\min}(A) \big)^{-1} \cdot \norm{E}_{2} \cdot \big(\lambda_{\min}(A + E) \big)^{-1} \nonumber \\
        \le& \frac {\norm{E}_2} {\lambda_{\min}(A) (\lambda_{\min}(A) - \norm{E}_2)}. \label{eq_err_1}
    \end{align}
    \end{proof}

Now we are ready to prove the main theorem.

\begin{proof}[Proof of Theorem \ref{thm_separate}]
    There are two steps in the derivation of this theorem.
    First, we examine the difference between equation \eqref{eq_theta} and the population average, given by:
    \begin{align}
        \bar{\theta}_{i, j} = \frac{1}{A_{i, j} \cdot m} \sum_{1\le k\le T:\, \set{i, j} \in S_k} \bignorms{\tilde \Sigma \cdot \Bigg( \frac 1 {\alpha} \sum_{l \in S_k} \tilde Y^{(l)} \Bigg) - \tilde Y^{(i)}}^2, \label{eq_bar_theta}
    \end{align}
    where $T = \binom{T}{\alpha}$ is the total number of possible subsets of $\set {1, 2, \dots, T}$ that have a fixed size $\alpha$, and $A_{i, j} = \binom{T-2}{\alpha-2}$ is the number of subsets of size $\alpha$ such that $i, j$ are both in the subset.
    Thus, $\ex{\theta_{i,j}} = \bar{\theta}_{i,j}$.
    We will first argue that the difference between $\theta_{i, j}$ and $\bar{\theta}_{i, j}$.
    Let $B = \sup_{1\leq i \leq k}\bignorm{ Y^{(i)}}_{\infty}$.
    We will show that for any $1 \le i, j \le T$, their difference satisfies:
    \begin{align}\label{eq_theta_dist}
        \bigabs{\theta_{i, j} - \bar{\theta}_{i, j}} \le {4B^2}\sqrt{\frac{ \log ( T / \delta ) }{ n }},
    \end{align}
    by the Chernoff bound and the union bound.
    
    Since $\tilde \Sigma^2 = \tilde \Sigma$, the eigenvalues of $\tilde \Sigma$ only consist of some ones and zeros. For any subset $S_k$, We have
    \begin{align*}
        \bignorms{\tilde \Sigma \cdot \Bigg( \frac 1 {\alpha} \sum_{l \in S_k} \tilde Y^{(l)} \Bigg) - \tilde Y^{(i)}}^2 &\leq \Bigg( \frac 1 {\alpha} \sum_{l \in S_k} \bignorms{\tilde \Sigma \cdot \tilde Y^{(l)} } + \bignorms{\tilde Y^{(i)}} \Bigg)^2 \\
        \leq& \Bigg( \frac 1 {\alpha} \sum_{l \in S_k} \bignorms{ \tilde Y^{(l)} } + \bignorms{\tilde Y^{(i)}} \Bigg)^2 \leq 4B^2 m.
    \end{align*}
    By the Chernoff bound, with probability $1 - 2\delta$ and for any $i,j$ and $k$, the gap between $\theta_{i, j}$ and $\bar \theta_{i,j}$ can be bounded as follows:
    \begin{align}
         \bigabs{\theta_{i, j} - \bar{\theta}_{i, j}} \le {4B^2}\sqrt{\frac{ \log ( \delta^{-1} ) }{2n_{i,j}}} \le 4B^2 \sqrt{\frac{\log(\delta^{-1})}{2 n}}.
    \end{align}
    By union bound, with probability $1 - 2 \binom{T}2 \cdot \delta$,
    we have
    \begin{align}
        \bigabs{\theta_{i, j} - \bar{\theta}_{i, j}} \le {4B^2}\sqrt{\frac{ \log ( \delta^{-1} ) }{2 n}}. 
    \end{align}
    Thus, the first step of proof has been completed.

    \bigskip
    In the second step, we want to compare $\theta_{i, j}$ with $\theta_{i, j'}$ where $i, j$ are from the same group, but $i, j'$ are from different groups.
    Based on equation \eqref{eq_sep}, we have that
    \begin{align*}
         \bar\theta_{i, j'} - \bar\theta_{i, j}
        = \frac 1 {A_{i, j'} \cdot m} \sum_{1\le k\le T:\,\set{i,j'}\subseteq S_k} \bignorms{\tilde \Sigma \cdot \Bigg(\frac 1 {\alpha} \sum_{l \in S_k} \tilde Y^{(l) } - \tilde Y^{(i)} \Bigg)}^2 
        - \frac 1 {A_{i, j} \cdot m} \sum_{1\le k\le T:\,\set{i,j}\subseteq S_{k}} \bignorms{\tilde \Sigma \cdot \Bigg(\frac 1 {\alpha} \sum_{l \in S_{k}} \tilde Y^{(l)} - \tilde Y^{(i)}\Bigg)}^2.
    \end{align*}
    Next, recall that $A_{i, j} = A_{i, j'} = \binom{\alpha - 2}{k - 2}$.
    Thus, the right-hand side of the above equation is equal to 
    \begin{align}
        \bar\theta_{i, j'} - \bar\theta_{i, j}
        = \frac 1 {A_{i, j} \cdot m} \Bigg( &\sum_{1\le k\le T:\,\set{i, j'}\subseteq S_k, j\notin S_k} \bignorms{\tilde \Sigma \cdot \Bigg(\frac 1 {\alpha} \sum_{l \in S_k} \tilde Y^{(l) } - \tilde Y^{(i)} \Bigg)}^2 
        - \sum_{1\le k\le T:\,\set{i, j}\subseteq S_{k},j'\notin S_{k}} \bignorms{\tilde \Sigma \cdot \Bigg(\frac 1 {\alpha} \sum_{l \in S_{k}} \tilde Y^{(l)} - \tilde Y^{(i)}\Bigg)}^2 \Bigg). \label{eq_diff}
    \end{align}
    Based on Claim \ref{claim_err1} and Claim \ref{claim_err2}, with probability at least $1 - \delta$, 
    \begin{align}
        \bignorms{\tilde \Sigma \Big(\tilde Y^{(i)} - \tilde Y^{(j)}\Big) - \Sigma \Big(Y^{(i)} - Y^{(j)}\Big)} 
        \leq &~ \Big( 2 B\sqrt d \log d \bignorms{P_{_G}} \Big) \frac{2 \sqrt d \log d \bignorms{P_{_G}} \log ( 2T^2 \delta^{-1} ) }{\sqrt m } \nonumber \\
        & + \Big(\sqrt d \log d \bignorms{P_{_G}} \Big) \frac{4 B \sqrt d \log d \bignorms{P_{_G}} \log ( 2T^2 \delta^{-1} ) }{\sqrt m }\nonumber \\
        =& \frac{8 B d \log^2 d \bignorms{P_{_G}}^2 \log ( 2T^2 \delta^{-1} ) }{\sqrt m }. \label{eq_con_err}
    \end{align}
    We can see that for any $k$ and $i,j$, with probability at least $1 - \delta$, the equation \eqref{eq_diff} is
    \begin{align*}
        \bar\theta_{i, j'} - \bar\theta_{i, j}
        \geq & \frac 1 {A_{i, j} \cdot m} \Bigg( \sum_{1\le k\le T:\,\set{i, j'}\subseteq S_k, j\notin S_k} \bignorms{\Sigma \cdot \Bigg(\frac 1 {\alpha} \sum_{l \in S_k} Y^{(l) } - Y^{(i)} \Bigg)}^2 
        - \sum_{1\le k\le T:\,\set{i, j}\subseteq S_{k},j'\notin S_{k}} \bignorms{\Sigma \cdot \Bigg(\frac 1 {\alpha} \sum_{l \in S_{k}} Y^{(l)} - Y^{(i)}\Bigg)}^2 \Bigg) \\
        & - 8B  \cdot \frac{8 B d \log^2 d \bignorms{P_{_G}}^2 \log ( 2T^2 \delta^{-1} ) }{\sqrt m } 
    \end{align*}
    Consider a subset $S^*$ where $i\in S^*$ and $j,j'\notin S^*$. $|S^*| = \alpha - 1$. We have
    \begin{align*}
        \Sigma \cdot \Bigg(\frac 1 {\alpha} \sum_{l \in S^* \cup j} Y^{(l) } - Y^{(i)} \Bigg) &=  \frac{1}{\alpha} \sum_{l \in S^*} \Sigma \cdot \Bigg( Y^{(l)} - Y^{(i)} \Bigg)    
        + \frac{1}{\alpha} \cdot \Sigma \cdot \Bigg( Y^{(j)} - Y^{(i)} \Bigg) \\
        \Sigma \cdot \Bigg(\frac 1 {\alpha} \sum_{l \in S^* \cup j'} Y^{(l) } - Y^{(i)} \Bigg) &=  \frac{1}{\alpha} \sum_{l \in S^*} \Sigma \cdot \Bigg( Y^{(l)} - Y^{(i)} \Bigg)    
        + \frac{1}{\alpha} \cdot \Sigma \cdot \Bigg( Y^{(j')} - Y^{(i)} \Bigg)
    \end{align*}
    Thus, we have 
    \begin{align*}
        &\bignorms{\Sigma \cdot \Bigg(\frac 1 {\alpha} \sum_{l \in S^* \cup j'} Y^{(l) } - Y^{(i)} \Bigg)}^2 - \bignorms{\Sigma \cdot \Bigg(\frac 1 {\alpha} \sum_{l \in S^* \cup j} Y^{(l) } - Y^{(i)} \Bigg)}^2 \\ 
        =& \frac{1}{\alpha^2} \Bigg( \bignorms{\Sigma \cdot \Big( Y^{(j')} - Y^{(i)} \Big) }^2 - \bignorms{\Sigma \cdot \Big( Y^{(j)} - Y^{(i)} \Big) }^2 \Bigg) 
        + 2 \cdot \Bigg \langle \frac{1}{\alpha} \sum_{l \in S^*} \Sigma \cdot \Big( Y^{(l)} - Y^{(i)} \Big), \frac{1}{\alpha} \cdot \Sigma \Big( Y^{(j')} - Y^{(i)} \Big) - \frac{1}{\alpha} \cdot \Sigma \Big( Y^{(j)} - Y^{(i)} \Big) \Bigg \rangle. 
    \end{align*}
    After taking the sum over all $S^*$, we obtain
    \begin{align*}
        \bar\theta_{i, j'} - \bar\theta_{i, j}
        \geq & \frac{1}{\alpha^2 m} \Bigg( \bignorms{\Sigma \cdot \Big( Y^{(j')} - Y^{(i)} \Big) }^2 - \bignorms{\Sigma \cdot \Big( Y^{(j)} - Y^{(i)} \Big) }^2 + \Bigg \langle 2 \frac{\binom{T-3}{\alpha - 3}}{\binom{T-2}{\alpha - 2}} \sum_{1\leq l \leq T: ~ l\neq i,j,j'} \Sigma \cdot \Big( Y^{(l)} - Y^{(i)} \Big), \Sigma \Big( Y^{(j')} - Y^{(j)} \Big) \Bigg \rangle \Bigg)\\
        & - \frac{64 B^2 d \log^2 d \bignorms{P_{_G}}^2 \log ( 2T^2 \delta^{-1} ) }{ \sqrt m } 
    \end{align*}
    Since $\binom{T-3}{\alpha - 3} / \binom{T-2}{\alpha - 2} = \frac{\alpha - 2}{T - 2}$, the above equation can be simplified as follows: 
    \begin{align*}
        \bar\theta_{i, j'} - \bar\theta_{i, j}
        \geq & \frac{1}{\alpha^2 m} \Bigg( b^2 - a^2 -\frac{2(\alpha - 2)}{T - 2} (T - 3) \cdot 4B^2 \Bigg) - \frac{64 B^2 d \log^2 d \bignorms{P_{_G}}^2 \log ( 2T^2 \delta^{-1} ) }{ \sqrt m } 
    \end{align*}
    From equation $\eqref{eq_theta_dist}$, we get
    \begin{align}
        \theta_{i, j'} - \theta_{i, j}
        \geq & \frac{1}{\alpha^2 m} \Bigg( b^2 - a^2 -\frac{2(\alpha - 2)}{T - 2} (T - 3) \cdot 4B^2 \Bigg) - \frac{64 B^2 d \log^2 d \bignorms{P_{_G}}^2 \log ( 2T^2 \delta^{-1} ) }{ \sqrt m } - 8B^2\sqrt{\frac{ \log ( T / \delta ) }{ n }} \label{eq_gap}
    \end{align}

    In conclusion, we have shown that provided that
    \begin{align}
        b^2 - a^2 \ge 8 \alpha B^2 m \varepsilon,
    \end{align} 
    then under the condition that
    \begin{align}
        & m \ge 64^2 B^4 d^2 \log^4 d \bignorms{P_{_G}}^4 \log^2(2T^2 \delta^{-1}) \varepsilon^{-2}, ~~\text{ and } \\
        & n \ge B^4 \log(T \delta^{-1})  \varepsilon^{-2},
    \end{align}
    the following holds:
         For any $i = 1, 2,\dots, T$, if $j'$ comes from the same group as $i$, but $j$ comes from a different group compared with $i$, then $\theta_{i, j'} - \theta_{i, j} \ge \varepsilon / 2$.
         
    Thus, we have shown that there exists a block structure in the affinity scores matrix $[\theta_{i, j}]_{T\times T}$.
    The proof of Theorem \ref{thm_separate} is now completed.
\end{proof}

\section{Additional Experiments}\label{sec_omitted_results}

In this section, we provide additional experiments to support our approach.
First, we provide additional observations regarding task relationships. These include structural differences between tasks in terms of PPR vectors and studying the effect of model size on mitigating negative transfer.
Second, we compare the running time of our approach with baselines.
Third, we provide comparisons with other baseline methods and ablation studies. 

\subsection{Comparing structural differences between tasks}\label{sec_ppr_similarity}

We study why negative transfers happen between two community assignment tasks when trained together. We hypothesize that the difference in graph diffusion processes of the two tasks contributes to negative transfer in graph neural networks. If their graph diffusion processes are very different, using a shared GNN would lead to a negative transfer. Otherwise, we expect a positive transfer between the two tasks. 
To validate this hypothesis, we compute each task's personalized PageRank (PPR) vector using its community labels as the seed set. Then, we measure the cosine similarity of the PPR vectors between tasks from the same group and tasks from different groups. We conduct this comparison based on the task groupings found by our method. 
Across four community detection datasets, tasks within the same group exhibit \textbf{8.8$\times$} higher cosine similarities in their PPR vectors than tasks from different groups. 
Table \ref{tab_ppr_similarity} reports the results.
This suggests that graph diffusion processes of tasks clustered together are more similar.

\begin{table}[h!]
\centering
\caption{This table reports the cosine similarity of PPR vectors between tasks from the same group and tasks from different groups. For tasks from the same group, their PPR vectors have higher cosine similarities than tasks from different groups.}
\label{tab_ppr_similarity}
\begin{tabular}{lcccc}
\toprule
Dataset & Amazon & Youtube & DBLP & LiveJournal \\
\midrule
PPR Cosine Similarity for tasks from the same group & 0.234 & 0.248 & 0.234 & 0.284 \\
PPR Cosine Similarity for tasks from different groups & 0.078 & 0.022 & 0.014 & 0.069 \\
\bottomrule
\end{tabular}
\end{table}

\subsection{Studying negative transfers by increasing the model size}\label{sec_model_size_GAMLP}
Previously in Section \ref{sec_observation}, we demonstrate that increasing the model size cannot address negative transfers between two tasks, using the SIGN model as the base model. We observe similar results of using a more powerful GNN with attention, i.e., GAMLP \citep{zhang2022graph}. 
We select one target task and one source task whose joint MTL performance is worse than STL for the target task. Then, We vary the hidden width of a GAMLP model from 32 to 2048 and repeat MTL training. Table \ref{tab_gamlp} reports the $F_1$ scores of the MTL and STL models on the target task. 
We observe that the MTL performance is consistently lower than STL, even as the width increases. 
This reaffirms our hypothesis that despite using larger models as width increases, they still do not fix the negative transfers from the source task to the target task in multi-task learning.

\begin{table*}[h!]
\centering
\caption{This table reports the $F_1$ score of a target task, comparing the STL and MTL with a source task that causes negative transfer. 
Even after increasing the model size of GAMLP, there exists a consistent negative transfer. }\label{tab_gamlp}
\begin{tabular}{@{}lccccccc@{}}
\toprule
Model size   & 32 & 64 & 128 & 256 & 512 & 1024 & 2048  \\ 
\midrule
MTL & 24.8 $\pm$ 0.3 & 25.5 $\pm$ 0.2 & 28.8 $\pm$ 0.3 & 32.4 $\pm$ 0.6 & 27.7 $\pm$ 0.2 & 26.5 $\pm$ 0.3 & 19.7 $\pm$ 0.1\\
STL & 46.3 $\pm$ 0.4 & 46.3 $\pm$ 0.2 & 47.0 $\pm$ 0.3  & 43.3 $\pm$ 0.3 & 42.9 $\pm$ 0.5 & 43.3 $\pm$ 0.3 & 43.6 $\pm$ 0.2 \\
\bottomrule
\end{tabular}
\end{table*}

\subsection{Comparing the running time} \label{sec_running_time}

We provide a comparison of the running time of our approach with the baselines in terms of GPU hours evaluated on a single GPU.
We notice that the running time of our approach is comparable to the naive MTL approach, using $3.5\times$ running time on average compared to the Naive MTL. 
This is achieved by using early stopping and downsampling in our implementation to speed up the training of each MTL model, i.e., computing $f_i(S)$ on each subset $S$.
The results are reported in Table \ref{tab_running_time}.

\begin{table}[h!]
\centering
\caption{Running time (GPU hours) of our approach after adding early stopping and downsampling to speed up the training of MTL models, as compared with baseline approaches.}\label{tab_running_time}
\begin{tabular}{lcccc}
\toprule
Method & Amazon & YouTube & DBLP & LiveJournal \\
\midrule
Naive MTL & 0.80H & 1.27H & 1.96H & 3.35H \\
Multi-Gate MoE & 7.01H & 10.39H & 17.28H & 31.74H \\
Forward Selection & 26.91H & 49.61H & 53.38H & 88.28H \\
Backward Selection & 37.90H & 62.89H & 69.39H & 105.94H \\
Clustering by First-Order Task Affinity & 92.68H & 199.99H & 224.22H & 305.69H \\
Clustering by Higher-Order Approximation & 87.35H & 197.26H & 207.56H & 294.69H \\
\midrule
Our approach w/o Early Stopping and Downsampling & 24.46H &	52.79H & 59.19H & 87.17H \\
Our approach w/ Early Stopping and Downsampling & 2.19H & 4.29H & 4.92H & 7.49H \\
\bottomrule
\end{tabular}
\end{table}

\subsection{Comparison with additional community detection and task grouping methods}\label{sec_additional_results}

We compare our approach with two other community detection methods, including MinCutPool \cite{bianchi2020spectral} and Deep Modularity Networks \cite{tsitsulin2023graph}. These methods design GNN pooling methods for graph clustering. For each method, we train the pooling module together with the SIGN base model. The results are shown in Table \ref{tab_first_order}. We find that our approach can also outperform them by \textbf{5.8\%} on average.

Furthermore, we compare our approach with clustering by two first-order task affinities, including the first-order task affinity \cite{fifty2021efficiently} and approximating higher-order task affinity through averaging \cite{standley2020tasks}. The results are shown in Table \ref{tab_first_order}. 
The results show that our approach outperforms them by \textbf{2.5\%} on average. 
This validates the advantage of using higher-order task affinities over first-order task affinities.

\begin{table*}[h!]
\centering
\caption{Macro F1-score of community detection tasks on four social networks. We compare our approach with GNN-based community detection methods and two previous task grouping methods based on first-order task affinity. 
Each result is averaged over three random seeds.
}\label{tab_first_order}
\begin{tabular}{@{}lcccc@{}}
\toprule
Dataset   & {Amazon} & {Youtube} & {DBLP} & {LiveJournal}  \\ 
\midrule
MinCutPool \cite{bianchi2020spectral}  & 84.24 $\pm$ 0.19 & 44.28 $\pm$ 0.49 & 67.49 $\pm$ 0.96 & 81.87 $\pm$ 1.06\\
Deep Modularity Networks \cite{tsitsulin2023graph}  & 83.30 $\pm$ 1.07 & 43.58 $\pm$ 0.77 & 66.32 $\pm$ 0.15 & 79.84 $\pm$ 0.80\\
Clustering by First-order Task Affinity \cite{fifty2021efficiently}  & 90.99 $\pm$ 4.06 & 45.23 $\pm$ 2.73 & 68.23 $\pm$ 3.24 & 83.76 $\pm$ 3.77 \\
Clustering by Higher-order Approximation \cite{standley2020tasks} & 91.61 $\pm$ 3.86 & 46.34 $\pm$ 2.57 & 68.87 $\pm$ 2.23 & 84.61 $\pm$ 2.56 \\
\midrule 
\textbf{Alg. \ref{alg_task_grouping} (Ours)} & \textbf{92.66 $\pm$ 4.85} & \textbf{49.62 $\pm$ 2.26} &  \textbf{70.68 $\pm$ 2.65}  & \textbf{88.43 $\pm$ 2.70} \\ 
\bottomrule
\end{tabular}
\end{table*}

\subsection{Ablation study of Algorithm \ref{alg_task_grouping}'s parameters}\label{sec_abl_param}
We discuss the three parameters in our approach: the number of task groups $b$, the subset size $\alpha$, and the number of subsets $n$. We find that a larger number of task groups yields better performance. Furthermore, our approach remains stable under the variation of subset size and number of subsets. The results are reported in Table \ref{tab_ablation}. 

First, we ablate the number of task groups $b$ and vary it between 5, 10, 20, and 100. 
The results confirm our hypothesis that a larger number of task groups leads to better performance. Moreover, $b=20$ yields comparable results to $b=100$, which achieves a 49.76 $F_1$ score. Therefore, we use $b=20$ in our experiments. 

Second, We vary the subset size  $\alpha$ between 5, 10, and 20.
We observe similar performance for different task grouping settings, using $\alpha = 10$ slightly outperforming the other two. 
The result suggests that using a larger $\alpha$ does not help because the number of related tasks in community detection applications is limited. 

Third, we vary the number of subsets $n$ between 1000, 1500, and 2000. 
We observe that using $n=1000$ still achieves comparable performance as using $n=2000$. The performance difference is within 0.5\%.

\begin{table*}[h!]
\centering
\caption{Study of parameter sensitivity on the Youtube Dataset by varying the number of clusters $b$, the subset size $\alpha$, and the number of samples $n$.}\label{tab_ablation}
\begin{tabular}{@{}lccccc@{}}
\toprule
$b$  & 5 & 10 & 20  \\
\midrule
Fix $\alpha=10$ and $n=2000$  & 46.65 $\pm$ 3.26 & 48.10 $\pm$ 3.09	& 49.62 $\pm$ 2.62  \\
\midrule
$\alpha$  & 5 & 10 & 20 \\ 
\midrule
Fix $n=2000$ and $b=20$ & 49.02 $\pm$ 3.10 & 49.62 $\pm$ 2.62 &  49.06 \\
\midrule
$n$  & 1000 & 1500 & 2000  \\
\midrule
Fix $\alpha=10$ and $b=20$  & 49.20 $\pm$ 2.99 & 49.42 $\pm$ 3.20	& 49.62 $\pm$ 2.62  \\
\bottomrule
\end{tabular}
\end{table*}

{}

\end{document}